\documentclass{article}
\usepackage{ijcai21}

\usepackage{times}
\usepackage{soul}
\usepackage{url}
\usepackage[hidelinks]{hyperref}
\usepackage[utf8]{inputenc}
\usepackage[small]{caption}
\usepackage{graphicx}
\usepackage{amsmath}
\usepackage{amsthm}
\usepackage{booktabs}
\usepackage{algorithm}
\usepackage{algorithmic}
\usepackage{epsfig}
\usepackage{amsmath}
\usepackage{amssymb}
\usepackage{bbm}

\usepackage{bm}
\usepackage{amsfonts,multirow,bigstrut,booktabs,ctable,latexsym}
\usepackage{subfig}
\usepackage{mathtools}
\usepackage[mathscr]{eucal}
\usepackage{verbatim}
\usepackage{amsthm}
\usepackage{algorithm}
\usepackage{color}
\usepackage{xcolor}
\usepackage{physics}
\usepackage{soul}
\usepackage{xcolor}
\usepackage{pgfpages}

\newtheorem{theorem}{Theorem}

\newtheorem{assumption}{Assumption}
\newtheorem{lemma}{Lemma}
\newtheorem{definition}{Definition}

\DeclareMathOperator{\randk}{rand} 

\title{Federated Learning with Sparsification-Amplified Privacy and Adaptive Optimization}

\author{
Rui Hu
\and
Yanmin Gong\footnote{Corresponding author}\And
Yuanxiong Guo
\affiliations
The University of Texas at San Antonio\\
\emails
\{rui.hu, yanmin.gong, yuanxiong.guo\}@utsa.edu
}

\begin{document}

\maketitle

\begin{abstract}
Federated learning (FL) enables distributed agents to collaboratively learn a centralized model without sharing their raw data with each other. However, data locality does not provide sufficient privacy protection, and it is desirable to facilitate FL with rigorous differential privacy (DP) guarantee. Existing DP mechanisms would introduce random noise with magnitude proportional to the model size, which can be quite large in deep neural networks. In this paper, we propose a new FL framework with sparsification-amplified privacy. Our approach integrates random sparsification with gradient perturbation on each agent to amplify privacy guarantee. Since sparsification would increase the number of communication rounds required to achieve a certain target accuracy, which is unfavorable for DP guarantee, we further introduce acceleration techniques to help reduce the privacy cost. We rigorously analyze the convergence of our approach and utilize Renyi DP to tightly account the end-to-end DP guarantee. Extensive experiments on benchmark datasets validate that our approach outperforms previous differentially-private FL approaches in both privacy guarantee and communication efficiency.  
\end{abstract}

\section{Introduction}\label{sec:intro}

Federated learning (FL) is a new distributed learning paradigm that enables multiple agents to collaboratively learn a shared model under the orchestration of the cloud without sharing their local data~\cite{mcmahan2017communication}. By keeping data locally, FL is advantageous in privacy and communication efficiency compared with traditional centralized learning paradigm. However, recent inference attacks~\cite{fredrikson2015model,shokri2017membership} show that the local model updates shared between agents could also lead to privacy leakage, and it is desirable to protect the shared local model updates with rigorous privacy guarantee. 

To address this issue, several privacy-preserving framework have been proposed, among which differential privacy (DP)~\cite{dwork2014algorithmic} has become the de-facto standard due to its rigorous privacy guarantee and effectiveness in data analysis tasks~\cite{abadi2016deep,hu2020personalized,huang2019dp,guo2018practical,gong2016private}. General DP mechanisms, such as Gaussian or Laplacian mechanism, rely on the injection of carefully calibrated noise to the output of an algorithm directly. This poses new challenges to achieving DP in FL because the added noise is proportional to the model size which can be very large with modern deep learning neural networks (e.g., millions of model parameter), resulting in significantly degraded model accuracy. Under the local DP setting where the cloud is not fully trusted, the challenges become more prominent as all the local updates shared with the cloud need to be protected. 

Existing works on differentially-private machine learning either consider centralized DP \cite{abadi2016deep}, or rely on costly techniques such as secure multi-party computation \cite{truex2019hybrid} and shuffling via anonymous channels \cite{liu2020flame,erlingsson2019amplification} to remove the requirement of a trusted cloud and improve the model accuracy in local DP. How to better balance the model accuracy and privacy protection in FL efficiently remains largely unknown.

In this paper, we propose a novel differentially-private FL scheme, called Fed-SPA, to provide strong privacy guarantee in the local DP setting while maintaining high model accuracy. In light of the observation that local model updates in FL are largely sparse, we design a sparsification-coded DP mechanism that integrates gradient perturbation with random sparsification to amplify the privacy guarantee with little sacrifice on the model accuracy. Random sparsification transforms a large vector into a sparse one by keeping only a random subset of coordinates while setting other coordinates to zeros. As we will show in this paper, random sparsification not only introduces randomness to the scheme, but also reduces the sensitivity of the shared model updates with respect to raw data, thus resulting in smaller privacy loss at every communication round. 
Furthermore, as sparsification can slow down the convergence speed of learning algorithms and increase the total number of communication rounds, we propose to further reduce the end-to-end privacy loss by using convergence acceleration techniques to offset the negative impact of sparsification. We provide theoretical analysis to demonstrate the convergence of our scheme
and rigorous privacy guarantee. 

%

The main contributions of this paper are summarized below.
\begin{itemize}
    \item We propose to use sparsification for privacy amplification in FL while also improving communication efficiency. Previous works that consider both communication efficiency and DP treat them as two separate goals and solve them in an uncoordinated manner. 
    %
    %
    Unlike previous approaches, we focus on the interplay of those two goals and aim to kill two birds with one stone in this paper, i.e., use sparsification as a tool to improve DP and achieve communication efficiency at the same time. We theoretically analyze the impacts of sparsification on the utility-privacy trade-off and design a sparsification-coded DP mechanism for FL that provides stronger privacy guarantee with the same amount of random noise.  
    
    \item To further improve the utility-privacy trade-off, we integrate the sparsification-coded DP mechanisms with convergence acceleration techniques, which can reduce the number of required communication rounds and ensure faster convergence.  
    %
    %
    Specifically, we adapt an acceleration strategy similar to that of Adam optimizer to the FL setting and design an adaptive aggregation strategy on the cloud to reduce the number of communication rounds. The resulting scheme called Fed-SPA provides a better utility-privacy trade-off than the state-of-art differentially private FL methods. 
    
    \item We empirically evaluate our scheme on several benchmark datasets. The experiment results show that Fed-SPA significantly boosts the model accuracy, and at the same time saves more than $80\%$ of the bandwidth cost compared to the state-of-art approaches under the same DP guarantee. 
\end{itemize}


It is worth noting that our scheme aims to improve the utility-privacy trade-off while achieving communication efficiency. This distinguishes our paper from previous studies on communication efficient and differentially-private distributed learning that focus on ensuring privacy protection and achieving communication efficiency at the same time. For example, cpSGD~\cite{agarwal2018cpsgd} is a modified distributed SGD scheme which is private and communication-efficient via gradient quantization and binomial mechanism. However, since quantization does not provide any privacy amplification effects as sparsification, the utility-privacy trade-off is not improved in that approach.  


\section{Preliminaries}\label{sec:pre}
DP is a rigorous notion of privacy that has become the de-facto standard for measuring privacy risk. In the context of FL, DP ensures that the exchanged model updates are nearly the same regardless of the usage of a data sample. %
%
In this paper, we consider a relaxed DP definition called R{\'e}nyi differential privacy (RDP), which is strictly stronger than $(\epsilon,\delta)$-DP for $\delta >0$ and allows tighter composition analysis.  

\begin{definition}[$(\alpha, \rho)$-RDP]\label{rdp}
Given a real number $\alpha\in(1,+\infty)$ and privacy parameter $\rho\geq 0$, a randomized mechanism $\mathcal{M}$ satisfies $(\alpha, \rho)$-RDP if for any two neighboring datasets $D, D^{\prime}$ that differs in one record, the R{\'e}nyi $\alpha$-divergence between $\mathcal{M}(D)$ and $\mathcal{M}(D^{\prime})$ satisfies
\begin{equation*}
D_{\alpha}[\mathcal{M}(D)\|\mathcal{M}(D^{\prime})]:= \frac{1}{\alpha-1}\log\mathbb{E}\left[ \left(\frac{{\mathcal{M}(D)}}{{\mathcal{M}(D^\prime)}}\right)^{\alpha}\right]\leq \rho,
\end{equation*}
where the expectation is taken over the output of $\mathcal{M}(D^\prime)$. 
\end{definition}
\begin{lemma}[RDP Composition \cite{mironov2017renyi}]\label{lemma:rdp_comr} 
If $\mathcal{M}_1$ satisfies $(\alpha, \rho_1)$-RDP and $\mathcal{M}_2$ satisfies $(\alpha, \rho_2)$-RDP, then their composition $ \mathcal{M}_1 \circ \mathcal{M}_2$ satisfies $(\alpha, \rho_1+\rho_2)$-RDP.
\end{lemma}


\begin{lemma}[Gaussian Mechanism \cite{mironov2017renyi}]\label{lemma:gaussian_mechanism}
Let $h: \mathcal{D} \rightarrow \mathbb{R}^d$ be a vector-valued function over datasets. The Gaussian mechanism $\mathcal{M} =  h(D) + \mathbf{b}$ with $\mathbf{b}\sim\mathcal{N}(0, \sigma^2\mathbf{I}_d)$ satisfies $(\alpha,\alpha \phi^2(h)/2\sigma^2)$-RDP, where $ \phi(h)$ is the $L_2$ sensitivity of $h$ defined by $ \phi(h)=\sup_{D,D^\prime}\|h(D)-h(D^\prime)\|_2$ with $D,D^\prime$ being two neighboring datasets in $\mathcal{D}$.
\end{lemma}


\section{Fed-SPA: Federated Learning with Sparsification-Amplified Privacy and Adaptive Optimization}\label{sec:sys-mod}


\paragraph{Notation.} We use $[n]$ to denote the set of integers $\{1,2,\ldots,n\}$ with any positive integer $n$, and $[\cdot]_j$ to denote the $j$-th coordinate of a vector. Let $\norm{\cdot}$ be the $\ell_2$ vector norm.

\subsection{Problem Formulation}\label{subsec:form_fed}


A typical FL system consists of $n$ agents and a central server (e.g., the cloud). Each agent $i\in[n]$ has a local dataset with $m$ data samples, and all agents collaboratively train a global model $\bm{\theta}$ on the collection of their local datasets under the orchestration of the central server. The agents in FL aim to find the optimal global model $\bm\theta$ by solving the following empirical risk minimization problem while keeping their data locally:
\begin{align}\label{fed_obj}
\min_{\bm{\theta}  \in \mathbb{R}^d} f(\bm{\theta}) := \frac{1}{n}\sum_{i =1}^{n} f_{i}(\bm{\theta}),
\end{align}
where $f_{i}(\bm{\theta}) = \mathbb{E}_{z\sim \mathcal{D}_i}[l_i(\bm{\theta};z)]$ represents the loss function of $i$-th agent (possibly non-convex), $\mathcal{D}_i$ is the data distribution of $i$-th agent, and $z$ represents a data sampled from $\mathcal{
D}_i$. For $i\neq j$, the data distributions $\mathcal{D}_i$ and $\mathcal{D}_j$ may be very different. 

\paragraph{Threat Model.} Before elaborating the proposed solutions, we first define the following threat model considered in this paper. The adversary considered here can be the ``honest-but-curious'' aggregation server or agents in the system. The aggregation server will honestly follow the designed training protocol but are curious about agents' private data and may infer it from the shared messages. Furthermore, some agents can collude with the aggregation server or each other to infer private information about a specific victim agent. Besides, the adversary could also be the passive outside attacker. These attackers can eavesdrop all shared messages in the execution of the training protocol but will not actively inject false messages into or interrupt message transmissions. 

\subsection{Classic FL Algorithm: Federated Averaging}\label{subsec:fedavg}

As the most widely-used algorithm in the FL setting, Federated Averaging (FedAvg) \cite{mcmahan2017communication} solves \eqref{fed_obj} by selecting and distributing the current global model to a subset of agents, running multiple steps of SGD in parallel on the selected agents, and then aggregating the model updates from those agents to improve the global model iteratively. Specifically, FedAvg involves $T$ communication rounds, and each round consists of four stages: First, at the beginning of round $t \in \{0,\dots,T-1\}$, the server selects a subset of agents $\mathcal{W} \subseteq [n]$ to participate and sends them the latest global model $\bm{\theta}_{t}$.
Second, each agent $i \in \mathcal{W}$ initializes its local model $\bm{\theta}_i^{t,0}$ to be the global model $\bm{\theta}_{t}$ and then performs $\tau$ iterations of SGD on its local dataset as follows:
\begin{equation}\label{eqn:local-sgd}
\bm{\theta}_i^{t,s+1} = \bm{\theta}_i^{t,s} - \eta_l \bm{g}_i^{t,s}, \quad s = 0, 1, \ldots, \tau - 1 
\end{equation}
where $\eta_l$ is the local learning rate. Here, $\bm{g}_i^{t,s}:=(1/B)\sum_{z\in\xi_i^{t,s}}\nabla l(\bm{\theta}_i^{t,s},z)$ represents the stochastic gradient computed on a mini-batch $\xi_i^{t,s}$ of $B$ samples, which a unbiased estimate of $\nabla f_i(\bm{\theta}_i^{t,s})$. Third, each agent $i \in \mathcal{W}$ uploads the final local model update $\bm{\theta}_i^{t,\tau}-\bm{\theta}_{t}$ to the server. Fourth, the server aggregates the local model updates from all participating agents and improves the global model as
\begin{equation}\label{eqn:server_avg}
\bm{\theta}_{t+1} = \bm{\theta}_{t} +  \frac{1}{|\mathcal{W}|} \sum_{i\in\mathcal{W}} ({\bm{\theta}}_i^{t,\tau}-\bm{\theta}_{t}). 
\end{equation}
The same procedure repeats for the next 
round. 

\paragraph{Privacy and Communication Drawbacks.}
Although FedAvg avoids the direct information leakage by keeping data locally, the intermediate updates exchanged during the collaboration process such as $\bm{\theta}_i^{t,\tau} - \bm{\theta}_{t}$ and $\bm{\theta}_t$ could still leak private information about the local data as demonstrated in recent advanced attacks such as model inversion attacks \cite{fredrikson2015model} and membership attacks \cite{shokri2017membership}. Furthermore, in FedAvg, agents need to repeatedly upload local model updates (i.e., $\bm{\theta}_{t} - \bm{\theta}_i^{t,\tau}$) of large size (e.g., millions of model parameters for modern deep neural network models) to the server and download the newly-updated global model (i.e., $\bm{\theta}_t$) from the server in order to learn an accurate global model (e.g., $\sim1000$ rounds for running CNN on MINIST dataset or $\sim4000$ for LSTM on Shakespeare dataset to reach 99\% accuracy \cite{mcmahan2017communication}). Since the privacy loss is proportional to the model dimension and the number of communication rounds, the size of added DP noise could be very large in order to provide a strong local DP guarantee, which will degrade the model accuracy heavily. Besides, as the bandwidth between the server and agents could be rather limited in practice (e.g., wireless connection between the cloud and smartphones), especially for uplink transmissions, the overall communication cost could be very high. The above drawbacks motivate us to develop a new privacy-preserving and communication-efficient FL scheme.

\subsection{Proposed Fed-SPA Algorithm}\label{subsec:Fed-SPA}
In this subsection, we present Fed-SPA, our proposed FL scheme with the goal of improving user privacy protection and also communication efficiency while maintaining high model accuracy. To ensure easy integration into existing packages/systems, Fed-SPA follows the same overall structure of FedAvg but differs in the following two key aspects: 1) local models are updated at each agent using a variant of SGD, where the gradients are perturbed by our proposed sparsification-coded DP mechanism; and 2) the global model is updated adaptively instead of simple averaging to accelerate convergence at the server. The entire process of Fed-SPA is summarized in Algorithm~\ref{algorithm-1}. 

\paragraph{User-Side Sparsification-Coded DP Mechanism.} %
To address privacy and communication aspects simultaneously, we design a sparsification-coded DP mechanism, which integrates Gaussian mechanism and sparsification to reduce the privacy loss of each local iteration and the size of transmitted local model updates at each communication round. Specifically, we use the $\randk_k$ sparsifier to reduce the message size by a factor of $d/k$, which is defined as follows:
\begin{definition}[$\randk_k$ Sparsifier]
\label{def:sparsifier}
For parameter $k \in [d]$, the operator $\randk_k:\mathbb{R}^d \times \Omega_k \rightarrow \mathbb{R}^d$ is defined for a vector $\mathbf{x}\in\mathbb{R}^d$ as
\begin{equation}
[\randk_k(\mathbf{x},\omega)]_j:=
\begin{cases} 
[\mathbf{x}]_j, & \text{if } j\in\omega,\\ 
0, & \text{otherwise},
\end{cases}
\end{equation}
where $\Omega_k = {[d] \choose k}$ denotes the set of all $k$-element subsets of $[d]$.
We omit the second argument whenever it is chosen uniformly at random, i.e., $\omega\sim_{u.a.r}\Omega_k$.
\end{definition}
With the $\randk_k$ sparsifier, our sparsification-coded DP mechanism works as follows: at communication round $t$, each selected agent $i \in \mathcal{W}$ first generates its own sparsifier $\randk_k(\cdot)$ by randomly sampling a set of $k$ active coordinates $\omega_i^t$ before performing $\tau$ SGD updates, and then, at $s$-th local iteration, perturbs and sparsifies the stochastic gradient using Gaussian noise and the generated sparsifier $\randk_k(\cdot)$. We use the same sparsifier (with active set $\omega_i^t$) across all local iterations of communication round $t$ on agent  $i$ so that the transmitted model update $({\bm{\theta}}_i^{t,\tau}-\bm{\theta}_{t})$ is still a sparse vector (which has active coordinates $\omega_i^t$), preserving the communication efficiency benefit of sparsification. Let $p=k/d$ represent the compression ratio, the update rule at each agent is:
\begin{equation}\label{eqn:update_rule}
{\bm{\theta}}_i^{t,s+1} = \bm{\theta}_i^{t,s} - \eta_{l} S_i^t(\bm{g}_i^{t,s} + \bm{b}_i^{t,s}),
\end{equation}
where $S_i^t(\cdot) := (1/p)\randk_k(\cdot)$ is a scaled variant of $\randk_k(\cdot)$, and $\bm{b}_i^{t,s} $ is the noise sampled from the Gaussian distribution $ \mathcal{N}(0, \sigma^2\mathbf{I}_d)$. We use the scaled sparsifier $S_i^t(\cdot)$ so that the sparsified noisy gradient is an unbiased estimate of the true gradient. 




\paragraph{Server-Side Adaptive Update.} The sparsification-coded DP mechanism will inevitably slow down the convergence speed due to the increased variance of the stochastic gradient used in each iteration, and the privacy loss of each agent increases proportionally with the number of iterations according to the composition property of DP in Lemma~\ref{lemma:rdp_comr}. To improve the privacy, we speed up the convergence by updating the model in an adaptive manner similar as Adam~\cite{kingma2014adam} on the server. 
The adaptive optimizer like Adam can be modified by adding noise to their gradients to provide better DP guarantee, in terms of reducing the iterations needed to achieve a target model accuracy~\cite{yu2018improve}. 
However, there are two main constraints unique to deploying the adaptive optimizer on each agent in the FL setting. First, it is often the case that each agent participate only once or several times intermittently during the entire training process, and hence, replacing SGD with an adaptive optimizer at each agent during the local update stage will perform poorly due to the stale historical information such as the momentum in Adam. Second, maintaining the historical information on resource-constrained agents (e.g. smartphones) is 
costly for computation and storage resource.

To address the above issues, in Fed-SPA, the server, rather than the agents, will carry out the adaptive update without any additional communication. The server maintains two momentum vectors $\bm{u},\bm{v}\in\mathbb{R}^d$ which get updated at each round. Specifically, at round $t$, after $\tau$ local iterations, each agent $i\in \mathcal{W}$ uploads its model update $\Delta_i^t:={\bm{\theta}}_i^{t,\tau}-\bm{\theta}_{t}$ to the server to improve the global model as follows:
\begin{equation}\label{eqn:adam-server}
\begin{cases}
 \bm{u}_{t} = \beta_1 \bm{u}_{t-1} + (1-\beta_1)\sum_{i\in\mathcal{W}} {\Delta}_i^{t}/|\mathcal{W}|,\\
  \bm{v}_{t} = \beta_2 \bm{v}_{t-1} + (1-\beta_2) \bm{u}_{t}^2,\\
  \bm{\theta}_{t+1} = \bm{\theta}_t +  {\eta_g\bm{u}_{t}}/({\sqrt{\bm{v}_{t}} + \kappa}),
\end{cases}
\end{equation}
where $\beta_1, \beta_2 \in [0,1)$ are momentum parameters, $\eta_g$ is the global learning rate, and $\kappa$ controls the degree of adaptivity. Note that the math operations in \eqref{eqn:adam-server} are element-wise.


\begin{algorithm}[t]
\caption{The Fed-SPA Algorithm}
\label{algorithm-1}
\begin{algorithmic}[1]
\REQUIRE Initial model $\bm{\theta}_{0}$, initial momentums $[\bm{v}_{-1}]_j\geq \kappa^2,\forall j\in[d]$, $\bm{u}_{-1}=\mathbf{0}_d$, noise magnitude $\sigma$, number of rounds $T$, number of local iterations $\tau$, compression ratio $p$, momentum parameters $\beta_1, \beta_2$, learning rates $\eta_l,\eta_g$, and batch size $B$. 
\FOR{$t=0$ to $T-1$}
    \STATE Randomly selects a set of agents $\mathcal{W}$
    \STATE Broadcasts $\bm{\theta}_t$ to all agents in $\mathcal{W}$
    \FOR{each agent $i \in \mathcal{W}$ in parallel}
        \STATE Generate a new sparsifier $ S_i^t(\cdot)$
        \STATE $\bm{\theta}_i^{t,0} \gets \bm{\theta}_t$
        \FOR{$s=0$ to $\tau-1$}
            \STATE Compute a stochastic gradient $\bm{g}_i^{t,s}$ over a mini-batch $\xi_i^{t,s}$ of $B$ samples
            \STATE ${\bm{\theta}}_i^{t,s+1} \gets \bm{\theta}_i^{t,s} - \eta_l S_i^t(\bm{g}_i^{t,s} + \bm{b}_i^{t,s})$ where $\bm{b}_i^{t,s}\sim \mathcal{N}(0, \sigma^2\mathbf{I}_d)$
       \ENDFOR
       \STATE ${\Delta}_i^{t} \gets {\bm{\theta}}_i^{t, \tau} - \bm{\theta}_{t} $ and upload ${\Delta}_i^{t}$ to the server
    \ENDFOR
    \STATE $ \bm{u}_{t} \gets \beta_1 \bm{u}_{t-1} + (1-\beta_1)\sum_{i\in\mathcal{W}} {\Delta}_i^{t}/|\mathcal{W}|$
    \STATE $\bm{v}_{t} \gets \beta_2 \bm{v}_{t-1} + (1-\beta_2) \bm{u}_{t}^2$
    \STATE $ \bm{\theta}_{t+1} \gets \bm{\theta}_t + \eta_g {\bm{u}_{t}}/({\sqrt{\bm{v}_{t}} + \kappa} )$
    \ENDFOR
\end{algorithmic}
\end{algorithm}


\section{Main Theoretical Results}\label{sec:main_results}
In this section, we give the formal privacy guarantee and rigorous convergence analysis of Fed-SPA. 
Before stating our results, we make the following assumptions: 
\begin{assumption}[Smoothness]
\label{assp:smooth}
The local objective function $f_i$ is $L$-smooth, i.e., for any $i\in[n]$ and $\mathbf{x}, \mathbf{y}\in \mathbb{R}^d$, we have $f_i(\mathbf{y}) \leq f_i(\mathbf{x}) + \langle\nabla f_i(\mathbf{x}), \mathbf{y}-\mathbf{x}\rangle + (L/2)\|\mathbf{y}-\mathbf{x}\|^2$.
\end{assumption}
\begin{assumption}[Bounded Variance]\label{assp:bounded_divergence}
Let $\bm{g}_i$ be the stochastic gradient over the mini-batch sampled from the distribution $\mathcal{D}_i$. The function $f_i$ has a bounded local variance, i.e.,  $\mathbb{E}\|[\bm{g}_i]_j - [\nabla f_i(\mathbf{x})]_j \|^2 \leq \zeta_{l,j}^2$ for all $\mathbf{x}\in\mathbb{R}^d$, $j\in[d]$ and $i\in[n]$. Moreover, the global variance is bounded, i.e., $(1/n)\sum_{i=1}^{n}\mathbb{E}\|[\nabla f_i(\mathbf{x})]_j - [\nabla f(\mathbf{x})]_j\|^2\leq \zeta_{g,j}^2$ for all $\mathbf{x}\in\mathbb{R}^d$, $j\in[d]$ and $i\in[n]$. We also denote $\zeta_{l}^2:=\sum_{j=1}^{d}\zeta_{l,j}^2$ and $\zeta_{g}^2:=\sum_{j=1}^{d}\zeta_{g,j}^2$ for convenience.
\end{assumption}
\begin{assumption}[Bounded Gradient]
\label{assp:bounded_gradient_coord}
The loss function $l_i(\mathbf{x},z)$ has $G/\sqrt{d}$-bounded gradients, i.e., for any data sample $z$ from $\mathcal{D}_i$, we have $|[\nabla l_i(\mathbf{x},z)]_j|\leq {G}/\sqrt{d}$ for all $\mathbf{x}\in\mathbb{R}^d$, $j\in[d]$ and $i\in[n]$.
\end{assumption}

Assumption~\ref{assp:smooth} is standard and implies that the global loss function $f$ is also $L$-smooth. Assumption~\ref{assp:bounded_divergence} and Assumption~\ref{assp:bounded_gradient_coord} are fairly standard in non-convex optimization literature~\cite{reddi2020adaptive,kingma2014adam}. 
Assumption~\ref{assp:bounded_gradient_coord} characterizes the sensitivity of each coordinate of gradient $\nabla l(\mathbf{x}, z)$ and implies $\mathbb{E}\|\nabla l(\mathbf{x}, z) \|^2\leq {G^2}$, which can be enforced by the gradient clipping technique. 

\subsection{Privacy Analysis}\label{sec:privacy_analysis}
In this subsection, we discuss how the sparsification in our sparsification-coded DP mechanism amplifies the privacy and provide the end-to-end privacy guarantee of Fed-SPA. 

In our sparsification-coded DP mechanism, the $\randk_k$ sparsifier does not provide any DP guarantee by itself, but it amplifies the privacy guarantee provided by additive Gaussian noises. To analyze the end-to-end privacy, we first need to analyze the sensitivity of the sparsified gradient in Line 9, Algorithm \ref{algorithm-1}. With the $\randk_k$ sparsifier, the active coordinate set $\omega$ is chosen independent of data and does not leak privacy, and only the values of active coordinates $\{[\mathbf{x}]_j,j\in\omega\}$ contain private data information and need to be protected. Therefore, in our sparsification-coded DP mechanism, only values of active coordinates of the gradient are actually perturbed by the Gaussian noise. More precisely, let $\omega_i^t$ denote the active coordinate set for participating agent $i$ at round $t$, the sparsified noisy gradient at each local iteration can be represented as
$
S_i^t(\bm{g}_i^{t,s} + \bm{b}_i^{t,s})  = {[\bm{g}_i^{t,s} + \bm{b}_i^{t,s}]_{\omega_i^t}}/{p}
={[\bm{g}_i^{t,s}]_{\omega_i^t}/{p} + [\bm{b}_i^{t,s}]_{\omega_i^t}}/{p},
$
where we can observe that the amount of added noise is proportional to $k$, the number of active coordinates, and is reduced by a factor of $d/k$ compared with the standard Gaussian mechanism. 
In the following, we analyze the sensitivity of $[\bm{g}_i^{t,s}]_{\omega_i^t}$ and then compute the privacy guarantee after adding noise $[\bm{b}_i^{t,s}]_{\omega_i^t}$. {For agent $i$, given any two neighboring datasets ${\xi_i^{t,s}}$ and $\grave{\xi}_i^{t,s}$ that have the same size $B$ but differ in one data sample (e.g., $z\in{\xi_i^{t,s}}$ and $\grave{z}\in\grave{\xi}_i^{t,s}$)}. The $L_2$ sensitivity of $[\bm{g}_i^{t,s}]_{\omega_i^t}$ can be denoted as $\phi^2_{\omega_i^t} 
%
%
 = \max \| [\bm{g}_i^{t,s}]_{\omega_i^t} -  [\nabla f_i(\bm{\theta}_i^{t,s}, \grave{\xi}_i^{t,s})]_{\omega_i^t} \|^2
 = \max \| ({1}/{B})[\nabla l(\bm{\theta}_i^{t,s},z )- \nabla l(\bm{\theta}_i^{t,s},\grave{z})]_{\omega_i^t} \|^2$ and is upper-bounded by
 $2 p G^2/B^2$ following Assumption~\ref{assp:bounded_gradient_coord}. We observe that the sensitivity of $[\bm{g}_i^{t,s}]_{\omega_i^t}$ is proportional to the compression ratio $p$, reducing the privacy loss according to Lemma~\ref{lemma:gaussian_mechanism}.

Then, we give the end-to-end DP guarantee of Fed-SPA in Theorem~\ref{thm:privacy_loss} and provide the proof in
Appendix E. Given a fixed value of $\delta$, $\epsilon$ is computed numerically by searching an optimal $\alpha$ that minimizes $\epsilon$. We observe that the noise magnitude $\sigma$ is proportional to $p$, which implies that sparsification, i.e., when $p< 1$, can reduce the magnitude of Gaussian noise and hence improve the model accuracy.

\begin{theorem}[Privacy Guarantee]
\label{thm:privacy_loss}
Suppose the mini-batch $\xi_i^{t,s}$ is sampled without replacement at each iteration. Let $q := B/m$ be the data sampling rate. Let $I_i$ represent the number of rounds agent $i$ participates during the training. Under Assumption~\ref{assp:bounded_gradient_coord}, if ${\sigma^\prime}^2 = \sigma^2B^2/2pG^2\geq 0.7$, Fed-SPA achieves $(\epsilon,\delta)$-DP for agent $i$, where
\begin{equation*}
\epsilon = \frac{7 q^2 I_i \tau \alpha pG^2}{B^2\sigma^2}  + \frac{\log(1/\delta)}{\alpha-1},
\end{equation*}
for any $\alpha \leq (2/3)\sigma^2\log{(1/q\alpha(1+{\sigma^\prime}^2))} +1$ and $\delta\in(0,1)$.
\end{theorem}

%

\subsection{Convergence Analysis}\label{sec:convergence_analysis}
In this subsection, we present the convergence results of Fed-SPA for general loss functions satisfying Assumptions~\ref{assp:smooth}-\ref{assp:bounded_gradient_coord}. For ease of illustration, we assume \textit{full participation}, i.e., $|\mathcal{W}|=n$, and $\beta_1=0$, though our analysis can be easily extended to $\beta_1>0$ and \textit{partial participation} (i.e., $|\mathcal{W}|<n$, see
Appendix H for details). %
%
%
%
As $f(\cdot)$ may be non-convex, we study the gradient of the global model $\bm{\theta}_t$ as $t$ increases. We give the result in Theorem~\ref{thm:converge_nonconvex} and the proof in the full version.

\begin{theorem}[Convergence Result]\label{thm:converge_nonconvex} 
Let Assumptions~\ref{assp:smooth}-\ref{assp:bounded_gradient_coord} hold, and $L,G,\zeta_l,\zeta_g$ be as defined therein. Suppose the local learning rate satisfies $\eta_l \leq \min\{ 1/8L\tau, (1/8\tau)\min\{\kappa\sqrt{d}/G, (\kappa^2\sqrt{d}/G\eta_gL)^{1/2}\}\}$. Let $\zeta_{dp}^2 = (G^2 + \zeta_l^2)/p + pd\sigma^2$, then the iterates of Algorithm~\ref{algorithm-1} satisfy:
\begin{align*}
    \frac{1}{T}\sum_{t=0}^{T-1}\mathbb{E}\|\nabla f(\bm{\theta}_t)\|^2 = \mathcal{O}\left(({\sqrt{\beta_2}\eta_l\tau G/\sqrt{d} + \kappa}) (\Xi + \Xi^\prime)\right)
\end{align*}
with 
\begin{align*}
    &\Xi = \frac{ f({\bm{\theta}}_{0}) -  f^*}{\eta_l  \eta_g \tau T} + \frac{5\eta_l^2 \tau L^2}{2 \kappa}\left(\zeta_{dp}^2 + 6\tau \zeta_{g}^2\right),\\
    & \Xi^\prime = \left(\frac{\eta_g L}{2} + \frac{G}{\sqrt{d}}\right)\left[ \frac{4\eta_l}{n\kappa^2}\zeta_{dp}^2 + \frac{20\eta_l^3\tau^2 L^2}{\kappa^2}\left(\zeta_{dp}^2 + 6\tau \zeta_{g}^2\right)\right],
\end{align*}
where $f^*$ is the optimal objective value.
\end{theorem}
We restate the above result for a specific choice of $\eta_l$, $\eta_g$ and $\kappa$ in Lemma~\ref{lemma:converge_rate} to highlight the dependence on $\tau$ and $T$. Note that when $T$ is sufficiently large compared to $\tau$, $\mathcal{O}(1/\sqrt{n\tau T})$ is the dominant term in Lemma~\ref{lemma:converge_rate}. Therefore, Fed-SPA converges at a rate of $\mathcal{O}(1/\sqrt{n\tau T})$, which matches the best known rate for the general non-convex setting of our interest \cite{reddi2020adaptive}. We also note that both sparsification and Gaussian noise will slow down the convergence. However, a small compression ratio $p$ can reduce the amount of added noise (i.e., the term $pd\sigma^2$) by a factor of $1/p$, which implies the privacy amplification effect of sparsification.

\begin{lemma}\label{lemma:converge_rate}
Choose the local learning rate $\eta_l = \Theta(1/L\tau \sqrt{T})$ that satisfies the condition in Theorem~\ref{thm:converge_nonconvex}. Suppose $\eta_g = \Theta(\sqrt{n\tau})$ and $\kappa=G/\sqrt{d}L$, then the iterates of Algorithm~\ref{algorithm-1} satisfy
\begin{align*}
    \frac{1}{T}\sum_{t=0}^{T-1}\mathbb{E}\|\nabla f(\bm{\theta}_t)\|^2 & = \mathcal{O}\left(\frac{f({\bm{\theta}}_{0}) -  f^*}{\sqrt{n\tau T}} + \frac{2\zeta_{dp}^2 L}{G^2\sqrt{n\tau T}} \right.\\
    & \left.+ \frac{(\zeta_{dp}^2 + 6\tau \zeta_{g}^2)}{G\tau T} + \frac{(\zeta_{dp}^2 + 6\tau \zeta_{g}^2)L\sqrt{n}}{G^2\sqrt{\tau} T^{3/2}}\right),
\end{align*}
when $T$ is sufficiently large.
\end{lemma}

\section{Experimental Evaluation}\label{sec:exp}
The goal of this section is to evaluate the performance of Fed-SPA on several benchmark datasets. We aim at evaluating the performance of Fed-SPA with different levels of compression and comparing them with the performance of the following three schemes: 1) FedAvg: the classic FL algorithm; 2) DP-Fed: this baseline follows the algorithm of FedAvg except that the stochastic gradient is perturbed by adding Gaussian noise $\mathcal{N}(0,\sigma^2\mathbf{I}_d)$; 3) cpSGD \cite{agarwal2018cpsgd}: this baseline follows the algorithm of classic distributed SGD except that the stochastic gradient is quantized into some discrete domain and then perturbed using Binomial mechanism.


\subsection{Experimental Setup}\label{subsec:exp_setup}
We explore two widely-used benchmark datasets in FL: MNIST \cite{lecun1998gradient} and CIFAR-10 \cite{krizhevsky2009learning}. The MNIST dataset consists of 10 classes of $28\times28$ handwritten digit images. There are 60K training examples and 10K testing examples, which are partitioned among $100$ agents, each containing 600 training and 100 testing examples. The CIFAR-10 dataset consists of 10 classes of $32\times32$ images. There are 50K training examples and 10K testing examples in the dataset, which are partitioned into 100 agents, each containing 500 training and 100 testing examples. We use a CNN model for the MNIST dataset, which has two $5\times5$ convolutional layers (the first with 10 filters, the second with 20 filters, each followed with $2\times2$ max pooling and ReLu activation), a fully connected layer with 50 units and ReLu activation, and a final softmax output layer. For the CIFAR-10 dataset, we use a CNN model that consists of three $3\times3$ convolution layers (the first with 64 filters, the second with 128 filters, the third with 256 filters, each followed with $2\times2$ max pooling and ReLu activation), two fully connected layers (the first with 128 units, the second with 256 units, each followed with ReLu activation), and a final softmax output layer.

We set the privacy failure probability $\delta = 10^{-3}$ and the sampling ratio of agents $r=|\mathcal{W}|/n=0.1$ for all experiments by default. 
Since cpSGD follows the classic distributed SGD scheme, we have $\tau=1$ and $r=1.0$ for cpSGD. 
We tune the hyperparameters using grid-search. We set the number of local iterations $\tau=300$ for MNIST and $\tau=50$ for CIFAR-10. The details of other hyperparameter settings are given in the full version.
%
The per-coordinate sensitivity $G/\sqrt{d}$ is selected during an initialization round for each scheme by taking the median value over $N$ absolute values. 

\subsection{Experimental Results}
Table~\ref{tab:mnist} shows the best accuracy over 45 communication rounds for each scheme on the MNIST dataset, and Table~\ref{tab:cifar} represents the best accuracy over 200 rounds for each scheme on the CIFAR-10 dataset. Assume each value of the model parameter is represented by a $32$-bit floating number. The \textit{Compression ratio} for Fed-SPA is calculated as $p=k/d$. For cpSGD, $p=\log_2(m+b)/32$, where $b$ implies the $b$-bit quantization and $m$ is the parameter for the Binomial distribution $\text{Bin}(m,0.5)$. For FedAvg and DP-Fed, we have $p=1.0$. \textit{Cost} is the average bandwidth consumption calculated as $p \times d \times 32 \times T \times r$, where $r$ is the sampling probability of devices and $T$ is the communication rounds. Note that since how to analyze Binomial mechanism in cpSGD using RDP is still unknown, we account the privacy loss $\epsilon$ of cpSGD using the standard subsampling amplification \cite{balle2018privacy} and advanced composition \cite{dwork2010boosting}.

\begin{table}[!t]
\renewcommand{\arraystretch}{1.2}
\begin{center}
\resizebox{0.45\textwidth}{!}{\begin{tabular}{c|c|c|c|c}
\hline
\multirow{2}{*}{Compression ratio} & \multirow{2}{*}{Algorithm} & \multicolumn{3}{c}{Performance}  \\
\cline{3-5} 
& &  Accuracy ($\%$) & Cost (MB) & $\epsilon$ \\\hline
\multirow{2}{*}{$p = 0.05$} 
& Fed-SPA & 92.65  & 0.0197 & 1.0 \\\cline{2-5} 
& cpSGD & diverge  & - & - \\\hline
\multirow{2}{*}{$p = 0.1$} 
& Fed-SPA & 92.16  & 0.0393 & 1.0 \\\cline{2-5} 
& cpSGD & diverge  & - & - \\\hline
\multirow{2}{*}{$p = 0.4$}  
& Fed-SPA & 94.84   & 0.1572 & 1.0 \\\cline{2-5} 
& cpSGD & 87.32  & 1.5720 & 3.63 \\\hline
\multirow{4}{*}{$p = 1.0$} 
& Fed-SPA & 94.70  & 0.3931 & 1.0 \\\cline{2-5}
& cpSGD & 88.48  & 3.9310 & 2.26 \\\cline{2-5}
& DP-Fed & 91.41  & 0.3931  & 1.0 \\ \cline{2-5} 
& FedAvg & 96.87 & 0.3931 & - \\
\hline
\end{tabular}}
\end{center}
\vspace*{-5pt}
\caption{Summary of results on MNIST dataset.}
\label{tab:mnist}
\end{table}

\begin{table}[!b]
\renewcommand{\arraystretch}{1.2}
\begin{center}
\resizebox{0.45\textwidth}{!}{\begin{tabular}{c|c|c|c|c}
\hline
\multirow{2}{*}{Compression ratio} & \multirow{2}{*}{Algorithm} & \multicolumn{3}{c}{Performance}  \\
\cline{3-5} 
& &  Accuracy ($\%$) & Cost (MB) & $\epsilon$ \\\hline
\multirow{2}{*}{$p = 0.05$} 
& Fed-SPA & 63.0  & - & 1.0 \\\cline{2-5} 
& cpSGD & diverge  & - & - \\\hline
\multirow{2}{*}{$p = 0.1$} 
& Fed-SPA & 63.28 & 2.15 & 1.0 \\\cline{2-5} 
& cpSGD & diverge  & - & - \\\hline
\multirow{2}{*}{$p = 0.4$}  
& Fed-SPA & 64.36  & 8.60 & 1.0 \\\cline{2-5} 
& cpSGD & 35.74 & 86.02 & 258 \\\hline
\multirow{4}{*}{$p = 1.0$} 
& Fed-SPA & 62.06 & 21.50 & 1.0 \\\cline{2-5}
& cpSGD & 37.17  & 215.04 & 39.77 \\\cline{2-5}
& DP-Fed & 61.13 & 21.50  & 1.0 \\ \cline{2-5} 
& FedAvg & 67.16 & 21.50 & - \\
\hline
\end{tabular}}
\end{center}
\vspace*{-5pt}
\caption{Summary of results on CIFAR-10 dataset.}
\label{tab:cifar}
\end{table}

Without compression, i.e., when $p=1.0$, Fed-SPA outperforms cpSGD and DP-Fed on both datasets as the server in Fed-SPA updates the global model adaptively to speed up the convergence, and its best accuracy to achieve $(1.0, 10^{-3})$-DP is close to the best accuracy of FedAvg under the same communication cost. Note that since the Binomial noise used in cpSGD is not as concentrated as the Gaussian noise used in Fed-SPA and DP-Fed, the model accuracy of cpSGD degrades heavily, and hence to achieve a target accuracy, the privacy loss of cpSGD is very large. As the compression ratio $p$ decreases, the performance of Fed-SPA on both datasets does not degrade much, compared with cpSGD which cannot achieve a reasonable privacy guarantee when $p\leq 0.1$. More importantly, Fed-SPA for MNIST achieves a higher accuracy than DP-Fed under the same privacy cost, while saving $86\%$ communication cost when $p=0.05$. The reason is that decreasing $p$ also lowers the sensitivity of the sparsified gradient which has a direct impact on the noise magnitude to achieve a target privacy guarantee, as explain in Section~\ref{sec:privacy_analysis}. We observe a similar trend for CIFAR-10 from Table~\ref{tab:cifar}: Fed-SPA performs better when the compression ratio is small, e.g., when $p \leq 0.4$, as the sensitivity in this case is small.  

\begin{figure}[!t]
\vspace*{-5pt}
\centering
\subfloat[MNIST.]{\includegraphics[width=0.49\columnwidth]{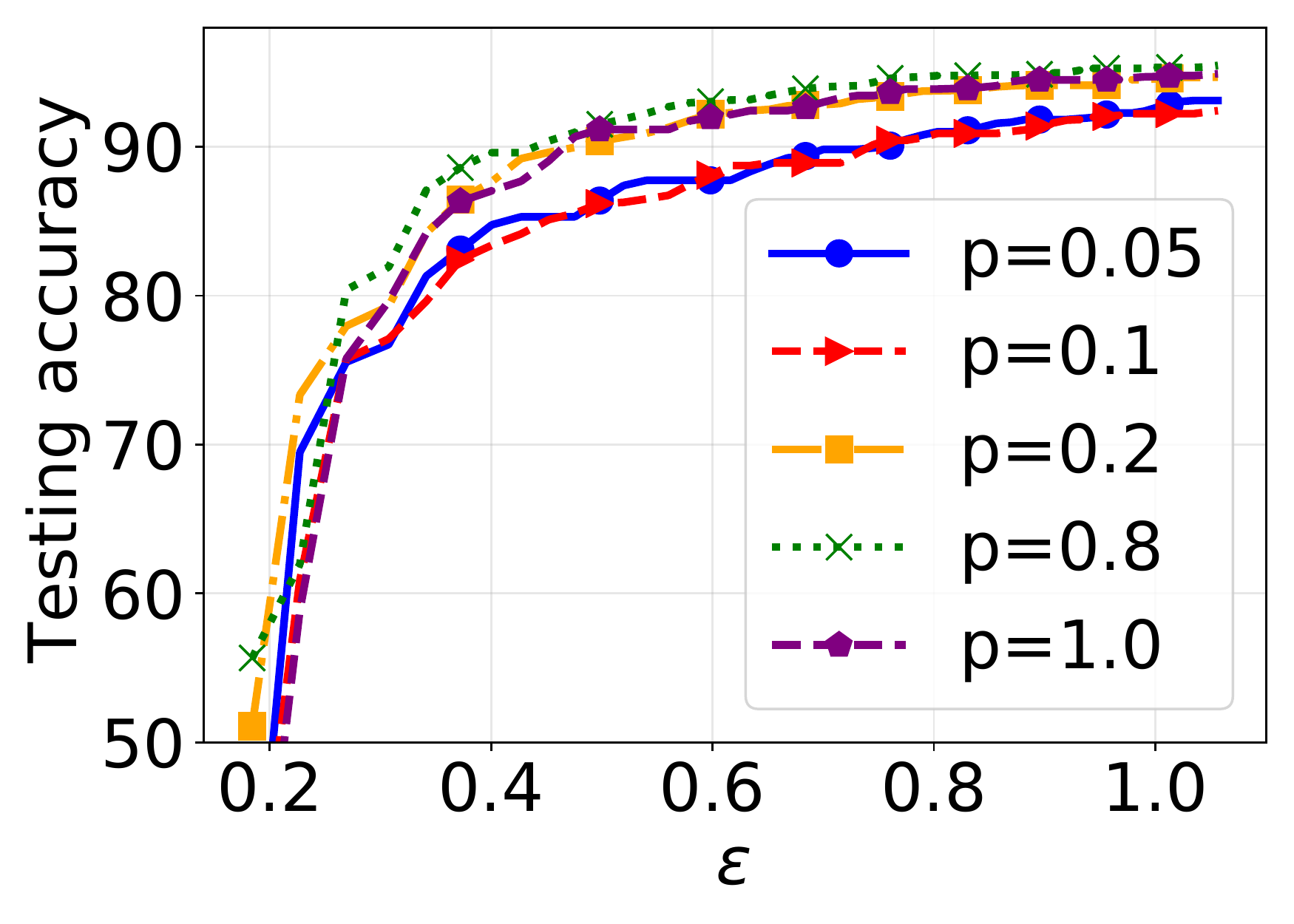}\label{fig:f1}}
\vspace*{-5pt}
\subfloat[CIFAR-10.]{\includegraphics[width=0.49\columnwidth]{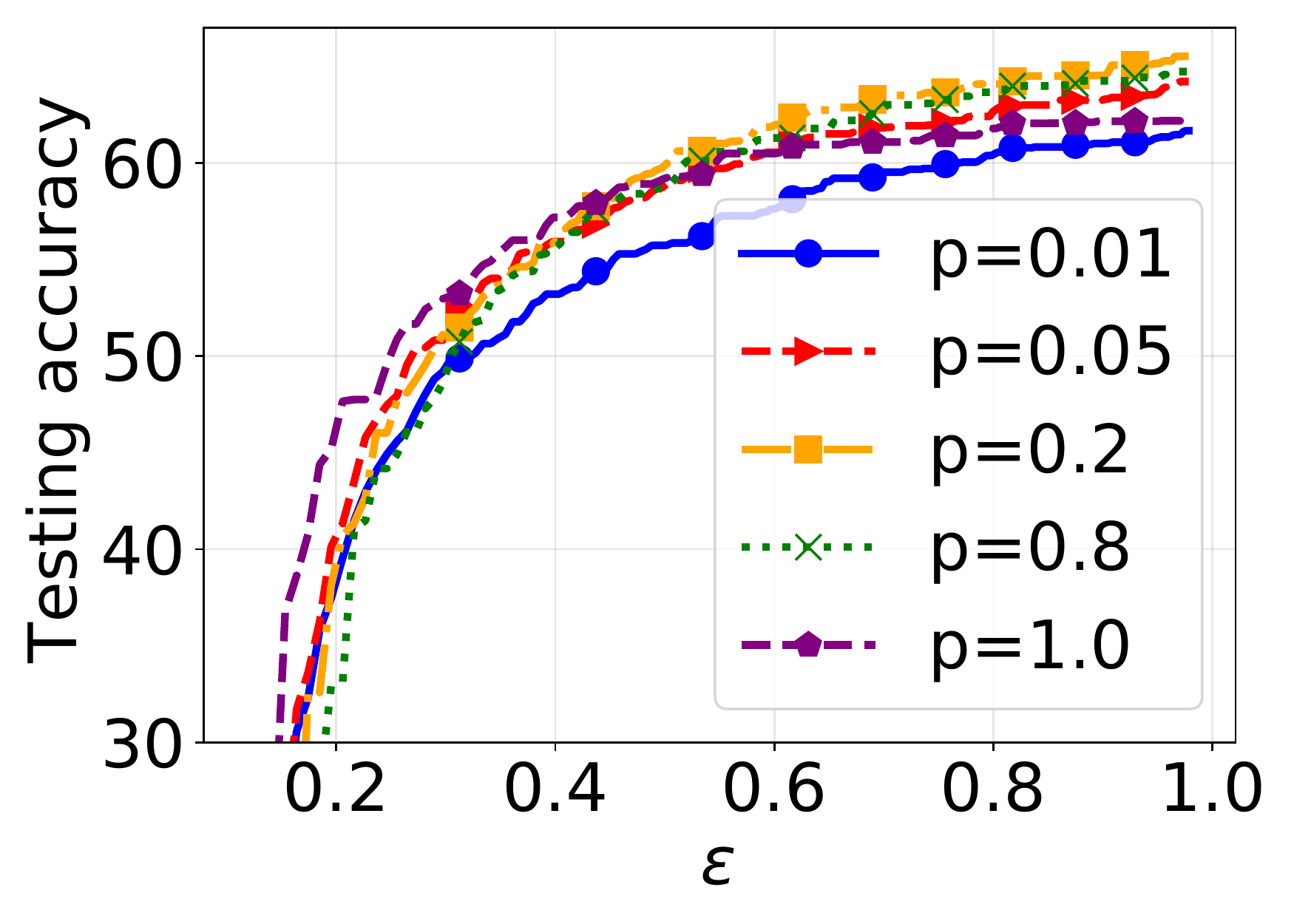}\label{fig:f2}}
\caption{Privacy-accuracy trade-off of Fed-SPA.}\label{fig:tradeoff}
\end{figure}

Then, we show the privacy-accuracy trade-off of Fed-SPA for both datasets with different levels of compression in Figure~\ref{fig:tradeoff}. As we mentioned above, a small compression ratio results in a smaller sensitivity, and hence less Gaussian noise will be added to the model. On the other hand, the compression ratio cannot be arbitrarily small as it will increase the variance of gradient as explained in Section~\ref{sec:convergence_analysis}. Therefore, we should find an optimal $p$ that is small enough to reduce the size of additive Gaussian noise but large enough to avoid large compression error. For both datasets, we can observe that when the privacy budget $\epsilon$ is limited (e.g., $\epsilon<1.0$), a small enough $p$ (e.g., $p=0.8$ for MNIST and $p=0.2$ for CIFAR-10) achieves higher accuracy as it reduces the size of Gaussian noise. As $\epsilon$ increases, the size of Gaussian noise decreases, and hence Fed-SPA with larger $p$ performs better due to the smaller compression error.

\section{Conclusion}\label{sec:con}

This paper has proposed Fed-SPA, a new FL scheme based on sparsification-coded DP mechanism and server-side adaptive update, to improve privacy-accuracy trade-off and achieve communication efficiency at the same time. We have provided rigorous convergence and privacy analysis of Fed-SPA. Extensive experiments based on benchmark datasets have been conducted to verify the effectiveness of the proposed scheme and numerically show the trade-off between privacy guarantee and model accuracy. For future work, we plan to investigate the interplay of privacy protection with other communication-efficient techniques in FL.

\section*{Acknowledgements}
The work of R. Hu and Y. Gong was supported by the U.S. National Science Foundation under grants US CNS-2029685 and CNS-1850523. 
The work of Y. Guo was supported by the U.S. National Science Foundation under grant US CNS-2029685.

\newpage

\bibliographystyle{named}
\bibliography{guo.bib,hu.bib,gong.bib}

\appendix
\onecolumn

\section{Experiments}
\subsection{Hyperparameter settings}\label{subsec:hyper}
We tune the hyperparameters using grid-search and choose the following optimal values: for MNIST, we use batch size $B=10$, number of local iterations $\tau=300$ and local learning rate $\eta_l=0.01$; 
for CIFAR-10, we use batch size $B=50$, number of local iterations $\tau=50$ and local learning rate $\eta_l=0.1$. 
for Fed-SPA, we set $\beta_1=0.9$, $\beta_2=0.99$ and $\kappa=10^{-3}$ for both datasets, and the global learning rate for MNIST and CIFAR-10 are $\eta_g=0.01$ and $\eta_g=0.005$, respectively. 
Note that we let both local and global learning rates decay at a rate of $1/\sqrt{t}$. 
%

\section{Additional Lemmas of RDP}

\begin{lemma}[RDP to DP conversion \cite{wang2018subsampled}]\label{lemma:rdp_dp}
If $\mathcal{M}$ satisfies $(\alpha, \rho)$-RDP, then it also satisfies $(\rho + \frac{\log(1/\delta)}{\alpha-1}, \delta)$-DP.
\end{lemma}

\begin{lemma}[RDP for Subsampling Mechanism \cite{wang2018subsampled,wang2019efficient}]\label{lemma:rdp_sub}
For a Gaussian mechanism $\mathcal{M}$ and any m-datapoints dataset $D$, define $\mathcal{M}\circ \textit{SUBSAMPLE}$ as 1) subsample without replacement $B$ datapoints from the dataset (denote $q=B/m$ as the sampling ratio); and 2) apply $\mathcal{M}$ on the subsampled dataset as input. Then if $\mathcal{M}$ satisfies $(\alpha,\rho(\alpha))$-RDP with respect to the subsampled dataset for all integers $\alpha \geq 2$, then the new randomized mechanism $\mathcal{M}\circ \textit{SUBSAMPLE}$ satisfies $(\alpha,\rho^\prime(\alpha))$-RDP with respect to $D$, where
\begin{multline*}
\rho^\prime(\alpha) \leq \frac{1}{\alpha-1}\log \bigg(1 + q^2 {\alpha \choose 2}\min\{4(e^{\rho(2)}-1), 2e^{\rho(2)}\}
+ \sum_{j=3}^{\alpha}q^j{\alpha \choose j}2e^{(j-1)\rho(j)}\bigg).
\end{multline*}
If ${\sigma^\prime}^2 = \sigma^2/\phi^2(h) \geq 0.7$ and $\alpha \leq (2/3)\sigma^2\log(1/q\alpha(1+{\sigma^\prime}^2)) + 1$, $\mathcal{M}\circ \textit{SUBSAMPLE}$ satisfies $(\alpha, 3.5 q^2 \phi^2(h) \alpha / \sigma^2)$-RDP. 
\end{lemma}

\section{Useful inequalities}

\begin{lemma}
For arbitrary set of $n$ vectors $\{\bm{a}_i\}_{i=1}^{n}, \bm{a}_i\in\mathbb{R}^d$,
\begin{equation}\label{lemma:aaaa}
    \left\|\sum_{i=1}^{n}\bm{a}_i\right\|^2\leq n{\sum_{i=1}^{n}\|\bm{a}_i\|^2}.
\end{equation}
\end{lemma}

\begin{lemma}
For given two vectors $\bm{a}, \bm{b}\in\mathbb{R}^d$,
\begin{equation}\label{lemma:a+b}
    \|\bm{a}+\bm{b}\|^2\leq (1+\alpha)\|\bm{a}\|^2 + (1+\alpha^{-1})\|\bm{b}\|^2, \forall \alpha>0.
\end{equation}
This inequality also holds for the sum of two matrices, $\mathbf{A},\mathbf{B}\in\mathbb{R}^{n\times d}$ in Frobenius norm.
\end{lemma}


\section{Proof of Lemma~\ref{lemma:unbiased-sparsifier}}\label{proof:unbiased-sparsifier}

\begin{lemma}[Unbiased Sparsification]\label{lemma:unbiased-sparsifier}
Given a vector $\mathbf{x}\in\mathbb{R}^d$, parameter $k \in [d]$, and sparsifier $S_i^t(\mathbf{x}) := (d/k) \randk_k(\mathbf{x})$. Let $p := k/d \in (0,1]$ be the compression ratio, it holds that 
\begin{equation*}
\mathbb{E}_{\omega}[S_i^t(\mathbf{x})]=\mathbf{x}, \quad \mathbb{E}_{\omega}\|S_i^t(\mathbf{x})-\mathbf{x}\|^2 = \left(\frac{1}{p}-1\right)\|\mathbf{x}\|^2.
\end{equation*}
\end{lemma}
\begin{proof}
By applying the expectation over the active set $\omega$, we have
\begin{align*}
    \mathbb{E}_{\omega}[S_i^t(\mathbf{x})] = \frac{d}{k}\left[\frac{k}{d}[\mathbf{x}]_1, \dots, \frac{k}{d}[\mathbf{x}]_d\right]=\mathbf{x},
\end{align*}
and
\begin{align*}
   \mathbb{E}_{\omega}\|S_i^t(\mathbf{x})-\mathbf{x}\|^2 & = \sum_{j=1}^{d}\left( \frac{k}{d} \left([\mathbf{x}]_j - \frac{d}{k}[\mathbf{x}]_j\right)^2 +  \left(1-\frac{k}{d}\right)[\mathbf{x}]_j^2\right)\\
   & = \left(\frac{1}{p}-1\right)\|\mathbf{x}\|^2.
\end{align*}
\end{proof}

\section{Proof of Theorem~\ref{thm:privacy_loss}}\label{proof:privacy_loss}
\begin{proof}
Let $\omega_i^t$ denote the active coordinate set for participating agent $i$ at round $t$. The sparsified noisy gradient at each local iteration can be represented as
\begin{align*}
S_i^t(\bm{g}_i^{t,s} + \bm{b}_i^{t,s}) & = \frac{[\bm{g}_i^{t,s} + \bm{b}_i^{t,s}]_{\omega_i^t}}{p}
=\frac{[\bm{g}_i^{t,s}]_{\omega_i^t} + [\bm{b}_i^{t,s}]_{\omega_i^t}}{p},
\end{align*}
where we can observe that due to the use of sparsification, Gaussian noises are only added to the values of active coordinates in $\omega_i^t$ in Fed-SPA. Therefore, we need to analyze the privacy loss of $[\bm{g}_i^{t,s}]_{\omega_i^t} $ after adding noise $[\bm{b}_i^{t,s}]_{\omega_i^t} $. For agent $i$, given any two neighboring mini-batch ${\xi_i^{t,s}}$ and $\grave{\xi}_i^{t,s}$ that have the same size $B$ but differ only in one data sample (e.g., $z$ and $\grave{z}$). Let $\phi_{\omega_i^t}$ be the $L_2$ sensitivity of $[\bm{g}_i^{t,s}]_{\omega_i^t}$, then we have 
\begin{align*}
\phi^2_{\omega_i^t} &= \max \left\| \left[\bm{g}_i^{t,s}\right]_{\omega_i^t} -  \left[\nabla f_i(\bm{\theta}_i^{t,s}, \grave{\xi}_i^{t,s})\right]_{\omega_i^t} \right\|^2\\
& = \max \left\| \frac{1}{B}\left[\nabla l(\bm{\theta}_i^{t,s},z )- \nabla l(\bm{\theta}_i^{t,s},\grave{z})\right]_{\omega_i^t} \right\|^2\\
& \leq \frac{2 p G^2}{B^2},
\end{align*}
where the last inequality comes from Assumption~\ref{assp:bounded_gradient_coord}. Given that the noise added to each coordinate of $ [\bm{g}_i^{t,s}]_{\omega_i^t}$ is drawn from the Gaussian distribution $\mathcal{N}(0,\sigma^2)$, each local iteration of Algorithm~\ref{algorithm-1} achieves $(\alpha, {\alpha pG^2}/{B^2\sigma^2} )$-RDP for the sub-sampled mini-batch $\xi_i^{t,s}$ by Lemma~\ref{lemma:gaussian_mechanism}.

By the sub-sampling amplification property of RDP in Lemma~\ref{lemma:rdp_sub}, each local iteration of Algorithm~\ref{algorithm-1} achieves $(\alpha, 7 q^2\rho_i(
\alpha))$-RDP for agent $i$'s local dataset $D_i$ if ${\sigma^\prime}^2=\sigma^2B^2/2pG^2\geq 0.7$ and $\alpha \leq (2/3)\sigma^2\log{(1/q\alpha(1+{\sigma^\prime}^2))}+1$.
%
%
After $T$ communication rounds, agent $i$ will perform $I_i \tau$ iterations of SGD. By Lemma \ref{lemma:rdp_comr}, Fed-SPA satisfies $(\alpha,7 q^2 I_i \tau \rho_i(
\alpha) )$-RDP for agent $i$ across $T$ communication rounds. Therefore, Theorem~\ref{thm:privacy_loss} follows by Lemma~\ref{lemma:rdp_dp}. 

For clarity of presentation, Theorem 1 relies on the amplification by sub-sampling result of \cite{wang2019efficient} which has a simple closed-form. A tighter and more general result (without restriction on the values of $\sigma$, $q$, $m$, $p$) can be readily obtained by using the results of \cite{wang2018subsampled}.
\end{proof}



\section{Proof of Theorem~\ref{thm:converge_nonconvex}}\label{proof:converge_nonconvex}
\begin{proof}
We note that both the moment and model update on the server are element-wise as following: $\forall j\in[d]$,
\begin{align*}
    &[\bm{u}_{t}]_{j} = \beta_1 [\bm{u}_{t-1}]_{j} + (1-\beta_1)[\Delta_{t}]_j,\\
    &[\bm{v}_{t}]_{j} = \beta_2 [\bm{v}_{t-1}]_{j} + (1-\beta_2)[\bm{u}_{t}]_{j}^2,\\
    &[\bm{\theta}_{t+1}]_j = [\bm{\theta}_t]_j + \eta_g \frac{[\bm{u}_{t}]_{j}}{\sqrt{[\bm{v}_{t}]_j} + \kappa}.
\end{align*}
By the $L$-smoothness of function $f$, we have
\begin{align}
    \nonumber
    \mathbb{E}_t[f(\bm{\theta}_{t+1}) -f(\bm{\theta}_t)] & \leq  \mathbb{E}_t\left\langle \nabla f(\bm{\theta}_{t}),   \bm{\theta}_{t+1} - \bm{\theta}_t\right\rangle + \frac{L}{2}\mathbb{E}_t\left\| \bm{\theta}_{t+1} - \bm{\theta}_t\right\|_2^2
    \\\nonumber
    & = \eta_g \mathbb{E}_t\left\langle \nabla f(\bm{\theta}_{t}),   \frac{\bm{u}_{t}}{\sqrt{\bm{v}_{t}} + \kappa} \right\rangle + \frac{L\eta_g^2}{2}\sum_{j=1}^{d}\mathbb{E}_t\left\| \frac{[\bm{u}_{t}]_{j}}{\sqrt{[\bm{v}_t]_j} + \kappa}\right\|_2^2\\\nonumber
    & =\eta_g \underbrace{ \mathbb{E}_t\left\langle \nabla f(\bm{\theta}_{t}),   \frac{\bm{u}_{t}}{\sqrt{\beta_2\bm{v}_{t-1}} + \kappa} \right\rangle}_{T_1} + \eta_g \underbrace{\mathbb{E}_t\left\langle \nabla f(\bm{\theta}_{t}),   \frac{\bm{u}_{t}}{\sqrt{\bm{v}_{t}} + \kappa} -\frac{\bm{u}_{t}}{\sqrt{\beta_2\bm{v}_{t-1}} + \kappa} \right\rangle}_{T_2} \\\label{eqn:a}
    & \ \ + \frac{L\eta_g^2}{2}\sum_{j=1}^{d}\mathbb{E}_t\left[ \frac{[\bm{u}_{t}]_{j}^2}{(\sqrt{[\bm{v}_t]_j} + \kappa)^2}\right]
\end{align}
where the expectation is taken over randomness at round $t$. 

\paragraph{Bounding $T_2$.} We observe the following about $T_2$:
\begin{align*}
    T_2 & = \mathbb{E}_t\left[ \sum_{j=1}^{d} [\nabla f(\bm{\theta}_{t})]_j \times \left[ \frac{[\bm{u}_{t}]_j}{\sqrt{[\bm{v}_t]_j} + \kappa} - \frac{[\bm{u}_{t}]_j}{\sqrt{\beta_2[\bm{v}_{t-1}]_j} + \kappa} \right] \right]\\
    & = \mathbb{E}_t\left[ \sum_{j=1}^{d} [\nabla f(\bm{\theta}_{t})]_j \times [\bm{u}_{t}]_j  \times \left[ \frac{\sqrt{\beta_2[\bm{v}_{t-1}]_j} - \sqrt{[\bm{v}_{t}]_j} }{(\sqrt{[\bm{v}_t]_j} + \kappa) (\sqrt{\beta_2[\bm{v}_{t-1}]_j} + \kappa) } \right] \right]\\
    & = \mathbb{E}_t\left[ \sum_{j=1}^{d} [\nabla f(\bm{\theta}_{t})]_j \times [\bm{u}_{t}]_j  \times \left[ \frac{ -(1-\beta_2)[\bm{u}_{t}]_j^2 }{(\sqrt{[\bm{v}_t]_j} + \kappa) (\sqrt{\beta_2[\bm{v}_{t-1}]_j} + \kappa) (\sqrt{\beta_2[\bm{v}_{t-1}]_j} + \sqrt{[\bm{v}_{t}]_j}) } \right] \right]\\
    & \leq (1-\beta_2) \mathbb{E}_t\left[ \sum_{j=1}^{d} |[\nabla f(\bm{\theta}_{t})]_j| \times |[\bm{u}_{t}]_j|  \times \left[ \frac{ [\bm{u}_{t}]_j^2 }{(\sqrt{[\bm{v}_t]_j} + \kappa) (\sqrt{\beta_2[\bm{v}_{t-1}]_j} + \kappa) (\sqrt{\beta_2[\bm{v}_{t-1}]_j} + \sqrt{[\bm{v}_{t}]_j}) } \right] \right]
\end{align*}
given that $ -(1-\beta_2)[\bm{u}_{t}]_j^2 =  (\sqrt{\beta_2[\bm{v}_{t-1}]_j} - \sqrt{[\bm{v}_{t}]_j})(\sqrt{\beta_2[\bm{v}_{t-1}]_j} + \sqrt{[\bm{v}_{t}]_j})$. Moreover, we have that
\begin{align*}
    T_2 & \leq (1-\beta_2) \mathbb{E}_t\left[ \sum_{j=1}^{d} |[\nabla f(\bm{\theta}_{t})]_j| \times \sqrt{\frac{(\sqrt{[\bm{v}_{t}]_j}-\sqrt{\beta_2[\bm{v}_{t-1}]_j} )(\sqrt{\beta_2[\bm{v}_{t-1}]_j} + \sqrt{[\bm{v}_{t}]_j})}{(1-\beta_2)} }\right. \\
    & \ \ \left. \times \left[ \frac{ [\bm{u}_{t}]_j^2 }{(\sqrt{[\bm{v}_t]_j} + \kappa) (\sqrt{\beta_2[\bm{v}_{t-1}]_j} + \kappa) (\sqrt{\beta_2[\bm{v}_{t-1}]_j} + \sqrt{[\bm{v}_{t}]_j}) } \right] \right]\\
    & \leq \sqrt{1-\beta_2} \mathbb{E}_t\left[ \sum_{j=1}^{d} |[\nabla f(\bm{\theta}_{t})]_j| \times \sqrt{\frac{(\sqrt{[\bm{v}_{t}]_j} - \sqrt{\beta_2[\bm{v}_{t-1}]_j})}{ (\sqrt{\beta_2[\bm{v}_{t-1}]_j} + \sqrt{[\bm{v}_{t}]_j}) }}  \times \left[ \frac{ [\bm{u}_{t}]_j^2 }{(\sqrt{[\bm{v}_t]_j} + \kappa) (\sqrt{\beta_2[\bm{v}_{t-1}]_j} + \kappa) } \right] \right]\\
    &  \leq \sqrt{1-\beta_2} \mathbb{E}_t\left[ \sum_{j=1}^{d} |[\nabla f(\bm{\theta}_{t})]_j| \times \left[ \frac{ [\bm{u}_{t}]_j^2}{(\sqrt{[\bm{v}_t]_j} + \kappa) (\sqrt{\beta_2[\bm{v}_{t-1}]_j} + \kappa) } \right] \right]
\end{align*}
using the fact that $[\bm{v}_{t-1}]_j $ is increasing in $t$ and ${(\sqrt{[\bm{v}_{t}]_j} - \sqrt{\beta_2[\bm{v}_{t-1}]_j})}/{ (\sqrt{\beta_2[\bm{v}_{t-1}]_j} + \sqrt{[\bm{v}_{t}]_j}) } \leq 1$. Next, using the fact that $|[\nabla f(\bm{\theta}_{t})]_j|\leq G/\sqrt{d}$ and $[\bm{v}_{t-1}]_j \geq \tau^2$ since $[\bm{v}_{-1}]_j\geq \tau^2$ and $ [\bm{v}_{t-1}]_j$ is increasing in $t$, one yields
\begin{align*}
    T_2 &\leq \sqrt{1-\beta_2}  \mathbb{E}_t\left[ \sum_{j=1}^{d}  \frac{G}{\kappa\sqrt{d}}\left[ \frac{ [\bm{u}_{t}]_j^2}{(\sqrt{[\bm{v}_t]_j} + \kappa) (\sqrt{\beta_2} + 1) } \right] \right]\\
    &\leq \sqrt{1-\beta_2}  \mathbb{E}_t\left[ \sum_{j=1}^{d}  \frac{G}{\kappa\sqrt{d}}\left[ \frac{ [\bm{u}_{t}]_j^2}{(\sqrt{[\bm{v}_t]_j} + \kappa)} \right] \right].
\end{align*}
For expository purpose, we assume $\beta_1=0$, so $\bm{u}_{t} = \Delta_{t}$ and hence 
\begin{align}\label{eqn:b}
    T_2 &\leq \sqrt{1-\beta_2}  \mathbb{E}_t\left[ \sum_{j=1}^{d}  \frac{G}{\kappa\sqrt{d}}\left[\frac{ [\Delta_t]_j^2}{(\sqrt{[\bm{v}_t]_j} + \kappa)} \right] \right].
\end{align}

\paragraph{Bounding $T_1$.} Note that
\begin{align}
\nonumber
    T_1 & = \sum_{j=1}^{d}\left\langle [\nabla f(\bm{\theta}_{t})]_j,   \mathbb{E}_t\left[\frac{[\bm{u}_{t}]_j}{\sqrt{\beta_2[\bm{v}_{t-1}]_j} + \kappa}\right] \right\rangle\\\nonumber
    & = \sum_{j=1}^{d}\left\langle \frac{[\nabla f(\bm{\theta}_{t})]_j}{\sqrt{\beta_2[\bm{v}_{t-1}]_j}+ \kappa},   \mathbb{E}_t\left[[\bm{u}_{t}]_j  - \eta_l\tau[\nabla f(\bm{\theta}_{t})]_j  + \eta_l\tau[\nabla f(\bm{\theta}_{t})]_j \right] \right\rangle\\\label{eqn:c}
    & = -\eta_l\tau\sum_{j=1}^{d} \frac{[\nabla f(\bm{\theta}_{t})]_j^2}{\sqrt{\beta_2[\bm{v}_{t-1}]_j}+ \kappa} + \underbrace{\left\langle \frac{\nabla f(\bm{\theta}_{t})}{\sqrt{\beta_2\bm{v}_{t-1}}+ \kappa},   \mathbb{E}_t\left[\bm{u}_{t}  + \eta_l\tau\nabla f(\bm{\theta}_{t}) \right] \right\rangle}_{T_3}.
\end{align}
The last term $T_3$ in \eqref{eqn:c} can be bounded as follows: as $\beta_1=0$, we have
\begin{align*}
    T_3 &= \left\langle \frac{\nabla f(\bm{\theta}_{t})}{\sqrt{\beta_2\bm{v}_{t-1}}+ \kappa},   \mathbb{E}_t\left[\Delta_t + \eta_l\tau\nabla f(\bm{\theta}_{t}) \right] \right\rangle\\
    & = \left\langle \frac{\nabla f(\bm{\theta}_{t})}{\sqrt{\beta_2\bm{v}_{t-1}}+ \kappa},   \mathbb{E}_t\left[-\frac{1}{n}\sum_{i=1}^{n}\sum_{s=0}^{\tau-1}\eta_l S_i^t(\bm{g}_i^{t,s} + \bm{b}_i^{t,s}) + \eta_l\tau\nabla f(\bm{\theta}_{t}) \right] \right\rangle\\
    & = \eta_l\tau\left\langle \frac{\nabla f(\bm{\theta}_{t})}{\sqrt{\beta_2\bm{v}_{t-1}}+ \kappa},   \mathbb{E}_t\left[-\frac{1}{n}\sum_{i=1}^{n}\frac{1}{\tau}\sum_{s=0}^{\tau-1} \nabla f_i(\bm{\theta}_i^{t,s}) + \nabla f(\bm{\theta}_{t}) \right] \right\rangle\\
    & =\frac{\eta_l\tau}{2} \sum_{j=1}^{d}\frac{[\nabla f(\bm{\theta}_{t})]_j^2}{\sqrt{\beta_2[\bm{v}_{t-1}]_j}+ \kappa} + \frac{\eta_l\tau}{2} \mathbb{E}_t\left\| \frac{1}{n}\sum_{i=1}^{n}\frac{1}{\tau}\sum_{s=0}^{\tau-1} \frac{\nabla f_i(\bm{\theta}_i^{t,s})}{\sqrt{\sqrt{\beta_2\bm{v}_{t-1}}+ \kappa}} - \frac{\nabla f(\bm{\theta}_{t})}{\sqrt{\sqrt{\beta_2\bm{v}_{t-1}}+ \kappa}}\right\|^2,
\end{align*}
where we use the fact that $ab\leq (a^2+b^2)/2$. Given that $  \nabla f(\bm{\theta}_{t}) = (1/n)\sum_{i=1}^{n} \nabla f_i(\bm{\theta}_{t})$, one yields
\begin{align}
\nonumber
    T_3 & \leq\frac{\eta_l\tau}{2} \sum_{j=1}^{d}\frac{[\nabla f(\bm{\theta}_{t})]_j^2}{\sqrt{\beta_2[\bm{v}_{t-1}]_j}+ \kappa} + \frac{\eta_l}{2\tau n^2} \mathbb{E}_t\left\|\sum_{i=1}^{n}\sum_{s=0}^{\tau-1} \frac{\nabla f_i(\bm{\theta}_i^{t,s}) - \nabla f_i(\bm{\theta}_{t})}{\sqrt{\sqrt{\beta_2\bm{v}_{t-1}}+ \kappa}}\right\|^2\\\nonumber
    & \leq \frac{\eta_l\tau}{2} \sum_{j=1}^{d}\frac{[\nabla f(\bm{\theta}_{t})]_j^2}{\sqrt{\beta_2[\bm{v}_{t-1}]_j}+ \kappa} + \frac{\eta_l}{2 n \kappa} \mathbb{E}_t\left[\sum_{i=1}^{n}\sum_{s=0}^{\tau-1} \left\|{\nabla f_i(\bm{\theta}_i^{t,s}) - \nabla f_i(\bm{\theta}_{t})}\right\|^2\right]\\\label{eqn:d}
    & \leq \frac{\eta_l\tau}{2} \sum_{j=1}^{d}\frac{[\nabla f(\bm{\theta}_{t})]_j^2}{\sqrt{\beta_2[\bm{v}_{t-1}]_j}+ \kappa} + \frac{\eta_l L^2}{2 n \kappa}\mathbb{E}_t\left[ \sum_{i=1}^{n}\sum_{s=0}^{\tau-1} \left\|\bm{\theta}_i^{t,s} - \bm{\theta}_{t}\right\|^2\right],
\end{align}
where the second inequality results from \eqref{lemma:aaaa}, and the third inequality uses the $L$-smoothness of $f_i$ and $[\bm{v}_{t-1}]_j > 0,j\in[d]$. Combining \eqref{eqn:c} and \eqref{eqn:d}, one get
\begin{align*}
    T_1 & \leq -\frac{\eta_l\tau}{2} \sum_{j=1}^{d}\frac{[\nabla f(\bm{\theta}_{t})]_j^2}{\sqrt{\beta_2[\bm{v}_{t-1}]_j}+ \kappa} + \frac{\eta_l L^2}{2 n \kappa}\mathbb{E}_t\left[ \sum_{i=1}^{n}\sum_{s=0}^{\tau-1} \left\|\bm{\theta}_i^{t,s} - \bm{\theta}_{t}\right\|^2\right].
\end{align*}
Using Lemma~\ref{lemma:distance}, we get
\begin{align}\label{eqn:e}
    T_1 & \leq -\frac{\eta_l\tau}{4} \sum_{j=1}^{d}\frac{[\nabla f(\bm{\theta}_{t})]_j^2}{\sqrt{\beta_2[\bm{v}_{t-1}]_j}+ \kappa} + \frac{5\eta_l^3 \tau^2 L^2}{2 \kappa}\left(\frac{1}{p}(G^2 + \zeta_l^2) + pd\sigma^2+ 6\tau \zeta_{g}^2\right). 
\end{align}
Combining \eqref{eqn:a}, \eqref{eqn:b} and \eqref{eqn:e}, we obtain that the loss gap at round $t$ is bounded as follows:
\begin{align*}
    \mathbb{E}_t[f(\bm{\theta}_{t+1}) -f(\bm{\theta}_t)] & \leq  -\frac{\eta_l\eta_g \tau}{4} \sum_{j=1}^{d}\frac{[\nabla f(\bm{\theta}_{t})]_j^2}{\sqrt{\beta_2[\bm{v}_{t-1}]_j}+ \kappa} + \frac{5\eta_g\eta_l^3 \tau^2 L^2}{2 \kappa}\left(\frac{1}{p}(G^2 + \zeta_l^2) + pd\sigma^2+ 6\tau \zeta_{g}^2\right)  \\
    & \ \ + \frac{G\eta_g \sqrt{1-\beta_2}}{\kappa\sqrt{d}}  \sum_{j=1}^{d}\mathbb{E}_t\left[\frac{ [\Delta_t]_j^2}{(\sqrt{[\bm{v}_t]_j} + \kappa)} \right] + \frac{L\eta_g^2}{2}\sum_{j=1}^{d}\mathbb{E}_t\left[ \frac{[\Delta_{t}]_{j}^2}{(\sqrt{[\bm{v}_t]_j} + \kappa)^2}\right].
\end{align*}
Summing over $t=0$ to $T-1$ and using telescoping sum, one can get
\begin{align*}
\nonumber
    \mathbb{E}[f(\bm{\theta}_{T}) -f(\bm{\theta}_0)] & \leq -\frac{\eta_l\eta_g\tau}{4}\sum_{t=0}^{T-1}\sum_{j=1}^{d}\frac{[\nabla f(\bm{\theta}_{t})]_j^2}{\sqrt{\beta_2[\bm{v}_{t-1}]_j}+ \kappa} + \frac{5\eta_g\eta_l^3 \tau^2 L^2}{2 \kappa}\left(\frac{1}{p}(G^2 + \zeta_l^2) + pd\sigma^2+ 6\tau \zeta_{g}^2\right) \\
    & \ \ + \frac{G\eta_g \sqrt{1-\beta_2}}{\kappa\sqrt{d}} \sum_{t=0}^{T-1} \sum_{j=1}^{d}\mathbb{E}\left[\frac{ [\Delta_t]_j^2}{(\sqrt{[\bm{v}_t]_j} + \kappa)} \right] + \frac{L\eta_g^2}{2}\sum_{t=0}^{T-1}\sum_{j=1}^{d}\mathbb{E}\left[ \frac{[\Delta_{t}]_{j}^2}{(\sqrt{[\bm{v}_t]_j} + \kappa)^2}\right].
\end{align*}

Using Lemma~\ref{lemma:delta} and the same idea in Lemma~\ref{lemma:delta} to bound $\sum_{t=0}^{T-1}\sum_{j=1}^{d}\mathbb{E}\left[\frac{ [\Delta_t]_j^2}{(\sqrt{[\bm{v}_t]_j} + \kappa)} \right]$, we finally obtain that
\begin{align}
\nonumber
    \mathbb{E}[f(\bm{\theta}_{T}) -f(\bm{\theta}_0)] & \leq -\frac{\eta_l\eta_g\tau}{4}\sum_{t=0}^{T-1}\sum_{j=1}^{d}\frac{[\nabla f(\bm{\theta}_{t})]_j^2}{\sqrt{\beta_2[\bm{v}_{t-1}]_j}+ \kappa} + \frac{5\eta_g\eta_l^3 \tau^2 L^2 T}{2 \kappa}\left(\frac{1}{p}(G^2 + \zeta_l^2) + pd\sigma^2+ 6\tau \zeta_{g}^2\right) \\\nonumber
    & \ \ + \left(\frac{G\eta_g \sqrt{1-\beta_2}}{\sqrt{d}} + \frac{\eta_g^2 L}{2}\right) \left[ \frac{4\eta_l^2\tau T}{n\kappa^2}\left(\frac{1}{p}(G^2 + \zeta_l^2) + pd\sigma^2\right) \right.\\\label{eqn:h}
    & \ \ \left. + \frac{20\eta_l^4\tau^3 L^2 T}{\kappa^2}\left(\frac{1}{p}(G^2 + \zeta_l^2) + pd\sigma^2+ 6\tau \zeta_{g}^2\right)\right]  + \left(\frac{G\eta_g \sqrt{1-\beta_2}}{\sqrt{d}} + \frac{\eta_g^2 L}{2}\right) \frac{4\eta_l^2\tau^2 }{\kappa^2} \sum_{t=0}^{T-1} \mathbb{E} \left\|\nabla f(\bm{\theta}_{t})\right\|^2.
\end{align}
Based on the condition of $\eta_l$ in Theorem~\ref{thm:converge_nonconvex}, we have
\begin{align*}
   \left(\frac{G\sqrt{1-\beta_2}}{\sqrt{d}} + \frac{\eta_g L}{2}\right) \frac{4\eta_l\tau }{\kappa^2} \leq  \frac{1}{8(\sqrt{\beta_2}\eta_l\tau G/\sqrt{d} + \kappa)} 
\end{align*}
We also observe that 
\begin{align*}
    \sqrt{\beta_2[\bm{v}_{t-1}]_j}+ \kappa \leq \sqrt{\beta_2}\eta_l\tau G/\sqrt{d} + \kappa,
\end{align*}
so rearranging \eqref{eqn:h} and dividing both sides by $T$ yields
\begin{align*}
    \frac{1}{T}\sum_{t=0}^{T-1}\mathbb{E}\left\|\nabla f(\bm{\theta}_{t})\right\|^2 & \leq \frac{8(\sqrt{\beta_2}\eta_l\tau G/\sqrt{d} + \kappa)}{T\eta_l\eta_g\tau} (f(\bm{\theta}_{0}) -f^*) + \frac{8(\sqrt{\beta_2}\eta_l\tau G/\sqrt{d} + \kappa)}{\eta_l\tau} \frac{5\eta_l^3 \tau^2 L^2}{2 \kappa}\left(\frac{1}{p}(G^2 + \zeta_l^2) + pd\sigma^2+ 6\tau \zeta_{g}^2\right) \\\nonumber
    & \ \ + \frac{8(\sqrt{\beta_2}\eta_l\tau G/\sqrt{d} + \kappa)}{\eta_l\tau}\left(\frac{G}{\sqrt{d}} + \frac{\eta_g L}{2}\right) \left[ \frac{4\eta_l^2\tau}{n\kappa^2}\left(\frac{1}{p}(G^2 + \zeta_l^2) + pd\sigma^2\right) \right.\\
    & \ \ \left. + \frac{20\eta_l^4\tau^3 L^2}{\kappa^2}\left(\frac{1}{p}(G^2 + \zeta_l^2) + pd\sigma^2+ 6\tau \zeta_{g}^2\right)\right],
\end{align*}
where we use the fact that $f(\bm{\theta}_T)\geq f^*$.

Note that in the above convergence analysis, we use $\beta_1 = 0$ just for expository purposes though our analysis can be directly extended to $\beta_1 > 0$. Specifically, if $\beta_1 > 0$, $\bm{u}_{t} = \beta_1 \bm{u}_{t-1} + (1-\beta_1)\Delta_t = (1-\beta_1) \sum_{h=1}^{t}\beta_1^{t-h}\Delta_h$ which only makes the expectation and variance of $\bm{u}_{t}$ different. Assume $\mathbb{E}[\Delta_h] = \mathbb{E}[\Delta_t]$, then we have $\mathbb{E}[\bm{u}_{t}] = (1-\beta_1^t)\Delta_t$ and $\sum_{t=1}^{T}\| \bm{u}_{t}\|^2 \leq \sum_{t=1}^{T}\|\Delta_t\|^2/(1-\beta_1)^2$. Then all the proof steps in Appendix F follow directly by substituting the above new expectation and variance bounds into $T_1$ and $T_2$.
\end{proof}

\section{Intermediate Results}

\begin{lemma}[Bounded Sparsified Noisy Gradient]\label{lemma:bounded_gradient}
The sparsified noisy gradient in Fed-SPA is unbiased and has a bounded variance, i.e., 
$ \mathbb{E}[S_i^t(\bm{g}_i^{t,s} + \bm{b}_i^{t,s})] =\nabla f_i(\bm{\theta}_i^{t,s})$ and $ \mathbb{E}\|S_i^t(\bm{g}_i^{t,s} + \bm{b}_i^{t,s})\|^2 \leq {(G^2+\zeta_l^2)}/{p} + pd\sigma^2$.
\end{lemma}
\begin{proof}
By the unbiasedness property of the sparsification in Lemma~\ref{lemma:unbiased-sparsifier}, we have
\begin{align*}
    \mathbb{E}[S_i^t(\bm{g}_i^{t,s} + \bm{b}_i^{t,s})] = \mathbb{E}_{\xi_i^{t,s}}[\bm{g}_i^{t,s}] + \mathbb{E}_{\bm{b}_i^{t,s}}[\bm{b}_i^{t,s}]=\nabla f_i(\bm{\theta}_i^{t,s}) ,
\end{align*}
where the expectation $\mathbb{E}[\cdot]$ is taken over the random variables $\omega_i^t,\xi_i^{t,s}$ and $\bm{b}_i^{t,s}$. Here, we use the fact that $\omega_i^t,\xi_i^{t,s}$ and $\bm{b}_i^{t,s}$ are independent with each other. Then, applying the expectation $\mathbb{E}[\cdot]$ on the norm of the sparsified noisy gradient, one yields
\begin{align*}
    \mathbb{E}\|S_i^t(\bm{g}_i^{t,s} + \bm{b}_i^{t,s})\|^2 &= \mathbb{E}\|S_i^t(\bm{g}_i^{t,s})\|^2 + \mathbb{E}\|[\bm{b}_i^{t,s}]_{\omega_i^t}\|^2\\
    & = \mathbb{E}\|S_i^t(\bm{g}_i^{t,s}) -  \nabla f_i(\bm{\theta}_i^{t,s}) \|^2 + \|\nabla f_i(\bm{\theta}_i^{t,s})\|^2 + \mathbb{E}_{\omega_i^t}\sum_{j\in\omega_i^t}\mathbb{E}_{[\bm{b}_i^{t,s}]_{j}}\|[\bm{b}_i^{t,s}]_{j}\|^2\\
    & =\mathbb{E}\|S_i^t(\bm{g}_i^{t,s}) - \bm{g}_i^{t,s} + \bm{g}_i^{t,s} - \nabla f_i(\bm{\theta}_i^{t,s})\|^2+ \|\nabla f_i(\bm{\theta}_i^{t,s})\|^2 + \frac{k}{d}\sum_{j\in[d]}\sigma^2
    \\
    & = \mathbb{E}\|S_i^t(\bm{g}_i^{t,s}) - \bm{g}_i^{t,s}\|^2 + \mathbb{E}\|\bm{g}_i^{t,s} - \nabla f_i(\bm{\theta}_i^{t,s})\|^2 + \|\nabla f_i(\bm{\theta}_i^{t,s})\|^2 + pd\sigma^2
    \\
    & \leq  \left(\frac{1}{p}-1\right) \mathbb{E}_{\xi_i^{t,s}}\|\bm{g}_i^{t,s}\|^2 + \zeta_l^2 + + \|\nabla f_i(\bm{\theta}_i^{t,s})\|^2 + pd\sigma^2
    \\
    & = \left(\frac{1}{p}-1\right) \mathbb{E}_{\xi_i^{t,s}}\|\bm{g}_i^{t,s}-\nabla f_i(\bm{\theta}_i^{t,s})\|^2 +  \left(\frac{1}{p}-1\right)\|\nabla f_i(\bm{\theta}_i^{t,s})\|^2 + \zeta_l^2 + \|\nabla f_i(\bm{\theta}_i^{t,s})\|^2 + pd\sigma^2\\
    & \leq  \frac{1}{p}\zeta_l^2 +  \frac{1}{p} \|\nabla f_i(\bm{\theta}_i^{t,s})\|^2 + pd\sigma^2\\
    & \leq \frac{1}{p}(G^2 + \zeta_l^2) + pd\sigma^2,
\end{align*}
using the fact that $\mathbb{E}[(X-\mathbb{E}[X])^2] = \mathbb{E}[X^2] -\mathbb{E}[X]^2$, Lemma~\ref{lemma:unbiased-sparsifier} and Assumption~\ref{assp:bounded_gradient_coord}.
\end{proof}

\begin{lemma}\label{lemma:distance}
The distance between the global model and the local model is bounded as follows: when $s\in[0,\tau)$,
\begin{align}
    \frac{1}{n}\sum_{i=1}^{n}\mathbb{E}\left\|\bm{\theta}_i^{t,s} - \bm{\theta}_{t}\right\|^2 \leq 5\tau\eta_l^2\left(\frac{1}{p}(G^2 + \zeta_l^2) + pd\sigma^2+ 6\tau \zeta_{g}^2\right) +  30\tau^2 \eta_l^2 \mathbb{E} \left\|\nabla f(\bm{\theta}_{t})\right\|^2,
\end{align}
if the local learning rate satisfies $\eta_l\leq 1/8L\tau$.
\end{lemma}
\begin{proof}
The result trivially holds for $s=0$ since $\bm{\theta}_i^{t,0}= \bm{\theta}_{t}$ for all $i\in[n]$. We observe that for any $s\in(0,\tau]$,
\begin{align*}
   \mathbb{E}\left\|\bm{\theta}_i^{t,s} - \bm{\theta}_{t}\right\|^2 
   & = \mathbb{E}\left\|\bm{\theta}_i^{t,s-1} -\eta_l S_i^t(\bm{g}_i^{t,s-1} +\bm{b}_i^{t,s}) - \bm{\theta}_{t}\right\|^2\\
   & = \mathbb{E}\left\|\bm{\theta}_i^{t,s-1} - \bm{\theta}_{t} -\eta_l S_i^t(\bm{g}_i^{t,s-1} +\bm{b}_i^{t,s}) + \eta_l \nabla f_i(\bm{\theta}_i^{t,s-1}) -\eta_l \nabla f_i(\bm{\theta}_i^{t,s-1}) + \eta_l \nabla f_i(\bm{\theta}_{t}) -\eta_l \nabla f_i(\bm{\theta}_{t}) \right. \\
   & \ \ \left. + \eta_l \nabla f(\bm{\theta}_{t}) - \eta_l \nabla f(\bm{\theta}_{t})\right\|^2\\
   & =  \mathbb{E}\left\|\bm{\theta}_i^{t,s-1} - \bm{\theta}_{t}  -\eta_l \nabla f_i(\bm{\theta}_i^{t,s-1}) + \eta_l \nabla f_i(\bm{\theta}_{t}) -\eta_l \nabla f_i(\bm{\theta}_{t}) + \eta_l \nabla f(\bm{\theta}_{t}) - \eta_l \nabla f(\bm{\theta}_{t})\right\|^2 \\
   & \ \ + \eta_l^2\mathbb{E}\left\|S_i^t(\bm{g}_i^{t,s-1} +\bm{b}_i^{t,s}) - \nabla f_i(\bm{\theta}_i^{t,s-1})\right\|^2\\
   & \leq (1+\alpha) \mathbb{E}\left\|\bm{\theta}_i^{t,s-1} - \bm{\theta}_{t}\right\|^2 + (1+\alpha^{-1}) \eta_l^2 \mathbb{E}\left\| - \nabla f_i(\bm{\theta}_i^{t,s-1}) +  \nabla f_i(\bm{\theta}_{t}) - \nabla f_i(\bm{\theta}_{t}) +  \nabla f(\bm{\theta}_{t}) -  \nabla f(\bm{\theta}_{t})\right\|^2 \\
   & \ \ + \eta_l^2\left(\mathbb{E}\left\|S_i^t(\bm{g}_i^{t,s-1} +\bm{b}_i^{t,s})\right\|^2 - \mathbb{E}\left\| \nabla f_i(\bm{\theta}_i^{t,s-1})\right\|^2 \right)\\
   & \leq \left(1+\frac{1}{2\tau-1}\right) \mathbb{E}\left\|\bm{\theta}_i^{t,s-1} - \bm{\theta}_{t}\right\|^2 + 6\tau \eta_l^2 \mathbb{E}\left\|\nabla f_i(\bm{\theta}_i^{t,s-1}) -  \nabla f_i(\bm{\theta}_{t})\right\|^2  + 6\tau \eta_l^2 \mathbb{E}\left\| \nabla f_i(\bm{\theta}_{t}) -  \nabla f(\bm{\theta}_{t})\right\|^2 \\
   & \ \ + 6\tau \eta_l^2 \mathbb{E}\left\|\nabla f(\bm{\theta}_{t})\right\|^2 +  \eta_l^2\left(\frac{1}{p}(G^2 + \zeta_l^2) + pd\sigma^2\right)  \\
   & \leq \left(1+\frac{1}{2\tau-1}\right) \mathbb{E}\left\|\bm{\theta}_i^{t,s-1} - \bm{\theta}_{t}\right\|^2 + 6\tau L^2 \eta_l^2\mathbb{E}\left\|\bm{\theta}_i^{t,s-1} -  \bm{\theta}_{t}\right\|^2 + 6\tau \eta_l^2 \zeta_{g}^2 + 6\tau \eta_l^2 \mathbb{E}\left\|\nabla f(\bm{\theta}_{t})\right\|^2 \\
   & \ \ + \eta_l^2\left(\frac{1}{p}(G^2 + \zeta_l^2) + pd\sigma^2\right) \\
   & = \left(1+\frac{1}{2\tau-1} +6\tau L^2 \eta_l^2 \right) \mathbb{E}\left\|\bm{\theta}_i^{t,s-1} - \bm{\theta}_{t}\right\|^2 + 6\tau \eta_l^2 \mathbb{E} \left\|\nabla f(\bm{\theta}_{t})\right\|^2 + \eta_l^2\left(\frac{1}{p}(G^2 + \zeta_l^2) + pd\sigma^2+ 6\tau \zeta_{g}^2\right),
\end{align*}
where the third equality uses the fact that $\mathbb{E}[S_i^t(\bm{g}_i^{t,s-1} +\bm{b}_i^{t,s})]=\nabla f_i(\bm{\theta}_i^{t,s-1})$, the first inequality results from \eqref{lemma:a+b}, the second inequality comes from choosing $\alpha=1/(2\tau-1)$, \eqref{lemma:aaaa} and Lemma~\ref{lemma:bounded_gradient}, and last inequality uses the $L$-smoothness of function $f_i$ and Assumption~\ref{assp:bounded_divergence}.

Taking the average over the agent $i$ and choosing the learning rate $\eta_l\leq 1/8L\tau$, we obtain that
\begin{align*}
    \frac{1}{n}\sum_{i=1}^{n}\mathbb{E}\left\|\bm{\theta}_i^{t,s} - \bm{\theta}_{t}\right\|^2 
    & \leq \left(1+\frac{1}{2\tau-1} +6\tau L^2 \eta_l^2 \right) \frac{1}{n}\sum_{i=1}^{n}\mathbb{E}\left\|\bm{\theta}_i^{t,s-1} - \bm{\theta}_{t}\right\|^2 + 6\tau \eta_l^2 \mathbb{E} \left\|\nabla f(\bm{\theta}_{t})\right\|^2 \\
    & \ \ + \eta_l^2\left(\frac{1}{p}(G^2 + \zeta_l^2) + pd\sigma^2+ 6\tau \zeta_{g}^2\right)\\
    & \leq \left(1+\frac{1}{\tau-1}\right) \frac{1}{n}\sum_{i=1}^{n}\mathbb{E}\left\|\bm{\theta}_i^{t,s-1} - \bm{\theta}_{t}\right\|^2 + 6\tau \eta_l^2 \mathbb{E} \left\|\nabla f(\bm{\theta}_{t})\right\|^2 + \eta_l^2\left(\frac{1}{p}(G^2 + \zeta_l^2) + pd\sigma^2+ 6\tau \zeta_{g}^2\right),
\end{align*}
where we use the fact that $1/(2\tau-1)\leq 1/2(\tau-1)$ and $ 6\tau L^2 \eta_l^2 \leq 1/2(\tau-1)$. Unrolling the recursion, one yields
\begin{align*}
    \frac{1}{n}\sum_{i=1}^{n}\mathbb{E}\left\|\bm{\theta}_i^{t,s} - \bm{\theta}_{t}\right\|^2 
    & \leq \sum_{h=0}^{s-1} \left(1+\frac{1}{\tau-1}\right)^h\left( 6\tau \eta_l^2 \mathbb{E} \left\|\nabla f(\bm{\theta}_{t})\right\|^2 + \eta_l^2\left(\frac{1}{p}(G^2 + \zeta_l^2) + pd\sigma^2+ 6\tau \zeta_{g}^2\right) \right)\\
    & \leq (\tau-1) \left[\left(1+\frac{1}{\tau-1}\right)^\tau -1\right] \left( 6\tau \eta_l^2 \mathbb{E} \left\|\nabla f(\bm{\theta}_{t})\right\|^2 +\eta_l^2\left(\frac{1}{p}(G^2 + \zeta_l^2) + pd\sigma^2+ 6\tau \zeta_{g}^2\right) \right)\\
    & \leq 5\tau\eta_l^2\left(\frac{1}{p}(G^2 + \zeta_l^2) + pd\sigma^2+ 6\tau \zeta_{g}^2\right) +  30\tau^2 \eta_l^2 \mathbb{E} \left\|\nabla f(\bm{\theta}_{t})\right\|^2,
\end{align*}
where we use the fact that $(1+1/(\tau-1))^\tau\leq 5$ for $\tau>1$.
\end{proof}

\begin{lemma}\label{lemma:delta}
The following upper bound holds for Algorithm~\ref{algorithm-1}:
\begin{align}
\nonumber
   \sum_{t=0}^{T-1}\sum_{j=1}^{d}\mathbb{E}\left[ \frac{[\Delta_{t}]_{j}^2}{(\sqrt{[\bm{v}_t]_j} + \kappa)^2}\right] & \leq  \frac{4\eta_l^2\tau T}{n\kappa^2}\left(\frac{1}{p}(G^2 + \zeta_l^2) + pd\sigma^2\right) + \frac{20\eta_l^4\tau^3 L^2 T}{\kappa^2}\left(\frac{1}{p}(G^2 + \zeta_l^2) + pd\sigma^2+ 6\tau \zeta_{g}^2\right) \\
   & \ \ +  \frac{4\eta_l^2\tau^2 }{\kappa^2} \sum_{t=0}^{T-1} \mathbb{E} \left\|\nabla f(\bm{\theta}_{t})\right\|^2.
\end{align}
\end{lemma}
\begin{proof}

We bound the desired quantity in the following manner:
\begin{align}
\nonumber
    \sum_{t=0}^{T-1}\sum_{j=1}^{d}\mathbb{E}\left[ \frac{[\Delta_{t}]_{j}^2}{(\sqrt{[\bm{v}_t]_j} + \kappa)^2}\right] & \leq \sum_{t=0}^{T-1}\sum_{j=1}^{d}\mathbb{E}\left[ \frac{[\Delta_{t}]_{j}^2}{[\bm{v}_t]_j + \kappa^2}\right]\\\nonumber
    &\leq \sum_{t=0}^{T-1}\sum_{j=1}^{d}\mathbb{E}\left[ \frac{[\Delta_{t}]_{j}^2}{\kappa^2}\right] \\\nonumber
    & \leq \frac{1}{\kappa^2}\sum_{t=0}^{T-1}\mathbb{E}\|\Delta_{t} + \eta_l\tau \nabla f(\bm{\theta}_{t}) -\eta_l\tau \nabla f(\bm{\theta}_{t})\|^2\\\label{eqn:g}
    & \leq \underbrace{\frac{2}{\kappa^2}\sum_{t=0}^{T-1}\mathbb{E}\|\Delta_{t} + \eta_l\tau \nabla f(\bm{\theta}_{t}) \|^2}_{T_4} + \frac{2\eta_l^2\tau^2}{\kappa^2}\sum_{t=0}^{T-1}\mathbb{E}\| \nabla f(\bm{\theta}_{t})\|^2.
\end{align}
Then, we bound $T_4$ as follows
\begin{align*}
    T_4 & = 2\sum_{t=0}^{T-1}\mathbb{E}\left\|\frac{1}{n}\sum_{i=1}^{n}\frac{1}{\kappa}\sum_{s=0}^{\tau-1}\eta_l\left(-S_i^t(\bm{g}_i^{t,s} + \bm{b}_i^{t,s}) + \nabla f_i(\bm{\theta}_i^{t,s}) - \nabla f_i(\bm{\theta}_i^{t,s}) + \nabla f_i(\bm{\theta}_{t}) - \nabla f_i(\bm{\theta}_{t}) + \nabla f(\bm{\theta}_{t})\right)\right\|^2\\
    & = {2\eta_l^2}\sum_{t=0}^{T-1}\mathbb{E}\left\|\frac{1}{n}\sum_{i=1}^{n}\frac{1}{\kappa}\sum_{s=0}^{\tau-1}(S_i^t(\bm{g}_i^{t,s} + \bm{b}_i^{t,s}) - \nabla f_i(\bm{\theta}_i^{t,s}) - \nabla f_i(\bm{\theta}_i^{t,s}) + \nabla f_i(\bm{\theta}_{t}))\right\|^2 \\
    & \leq \frac{4\eta_l^2}{n^2}\sum_{t=0}^{T-1}\mathbb{E} \left\|\sum_{i=1}^{n}\frac{1}{
    \kappa}\sum_{s=0}^{\tau-1}S_i^t(\bm{g}_i^{t,s} + \bm{b}_i^{t,s}) - \nabla f_i(\bm{\theta}_i^{t,s})\right\|^2 + \frac{4\eta_l^2}{n^2}\sum_{t=0}^{T-1}\mathbb{E}\left\|\sum_{i=1}^{n}\frac{1}{
    \kappa}\sum_{s=0}^{\tau-1}\nabla f_i(\bm{\theta}_i^{t,s}) - \nabla f_i(\bm{\theta}_{t})\right\|^2\\
   & \leq \frac{4\eta_l^2}{n^2\kappa^2}\sum_{t=0}^{T-1}\sum_{i=1}^{n}\sum_{s=0}^{\tau-1}\mathbb{E} \left\|S_i^t(\bm{g}_i^{t,s} + \bm{b}_i^{t,s}) - \nabla f_i(\bm{\theta}_i^{t,s})\right\|^2 + \frac{4\eta_l^2\tau }{n\kappa^2}\sum_{t=0}^{T-1}\sum_{i=1}^{n}\sum_{s=0}^{\tau-1}\mathbb{E}\left\|\nabla f_i(\bm{\theta}_i^{t,s}) - \nabla f_i(\bm{\theta}_{t})\right\|^2\\
    & \leq \frac{4\eta_l^2\tau T}{n\kappa^2}\left(\frac{1}{p}(G^2 + \zeta_l^2) + pd\sigma^2\right) + \frac{4\eta_l^2\tau L^2 }{n\kappa^2}\sum_{t=0}^{T-1}\sum_{i=1}^{n}\sum_{s=0}^{\tau-1}\mathbb{E}\left\|\bm{\theta}_i^{t,s}- \bm{\theta}_{t}\right\|^2\\
    & \leq \frac{4\eta_l^2\tau T}{n\kappa^2}\left(\frac{1}{p}(G^2 + \zeta_l^2) + pd\sigma^2\right) + \frac{4\eta_l^2\tau^2 L^2 }{\kappa^2}\sum_{t=0}^{T-1} 5\tau\eta_l^2\left(\frac{1}{p}(G^2 + \zeta_l^2) + pd\sigma^2+ 6\tau \zeta_{g}^2\right) +  30\tau^2 \eta_l^2 \mathbb{E} \left\|\nabla f(\bm{\theta}_{t})\right\|^2\\
    & =  \frac{4\eta_l^2\tau T}{n\kappa^2}\left(\frac{1}{p}(G^2 + \zeta_l^2) + pd\sigma^2\right) + \frac{20\eta_l^4\tau^3 L^2 T}{\kappa^2}\left(\frac{1}{p}(G^2 + \zeta_l^2) + pd\sigma^2+ 6\tau \zeta_{g}^2\right) +  \frac{2\eta_l^2\tau^2 }{\kappa^2} \sum_{t=0}^{T-1} \mathbb{E} \left\|\nabla f(\bm{\theta}_{t})\right\|^2,
\end{align*}
where the first inequality results from \eqref{lemma:aaaa}, the second inequality uses the independence of unbiased stochastic gradients $\{\bm{g}_i^{t,s}, i\in[n], s\in[0,\tau)\}$ and \eqref{lemma:aaaa}, and the last inequality comes from Lemma~\ref{lemma:distance}. The result follows by plugging $T_1$ into \eqref{eqn:g}.
\end{proof}

\section{Partial Participation}\label{proof:partial}
For the partial participation when $|\mathcal{W}| < n$, the main changed in the proof is the upper bound of $T_4$. The rest of the proof is similar to that of Theorem~\ref{thm:converge_nonconvex}, so we focus on bounding $T_4$ with partial participation here. Assume $\mathcal{W}$ is sampled uniformly for all subsets of $[n]$ with size $r$ at each round. In this case, we have 
\begin{align*}
    T_4 & = \frac{2}{\kappa^2}\sum_{t=0}^{T-1}\mathbb{E}\left\|\frac{1}{|\mathcal{W}|}\sum_{i\in\mathcal{W}}\sum_{s=0}^{\tau-1}\eta_l\left(-S_i^t(\bm{g}_i^{t,s} + \bm{b}_i^{t,s}) + \nabla f_i(\bm{\theta}_i^{t,s}) - \nabla f_i(\bm{\theta}_i^{t,s}) + \nabla f_i(\bm{\theta}_{t}) - \nabla f_i(\bm{\theta}_{t}) + \nabla f(\bm{\theta}_{t})\right)\right\|^2\\
    & \leq \frac{6\eta_l^2}{\kappa^2}\sum_{t=0}^{T-1}\mathbb{E}\left\|\frac{1}{r}\sum_{i\in\mathcal{W}}\sum_{s=0}^{\tau-1}(S_i^t(\bm{g}_i^{t,s} + \bm{b}_i^{t,s}) - \nabla f_i(\bm{\theta}_i^{t,s}))\right\|^2 + \frac{6\eta_l^2}{\kappa^2}\sum_{t=0}^{T-1}\mathbb{E}\left\|\frac{1}{r}\sum_{i\in\mathcal{W}}\sum_{s=0}^{\tau-1}\left( \nabla f_i(\bm{\theta}_i^{t,s}) - \nabla f_i(\bm{\theta}_{t})\right)\right\|^2 \\
    & \ \ +\frac{6\eta_l^2}{\kappa^2}\sum_{t=0}^{T-1}\mathbb{E}\left\|\frac{1}{r}\sum_{i\in\mathcal{W}}\sum_{s=0}^{\tau-1}\left( \nabla f_i(\bm{\theta}_{t}) - \nabla f(\bm{\theta}_{t}) \right)\right\|^2\\
    & \leq \frac{6\eta_l^2}{\kappa^2}\sum_{t=0}^{T-1}\mathbb{E}\left[ \frac{\tau}{r^2}\sum_{i\in\mathcal{W}}\sum_{s=0}^{\tau-1}\left\|S_i^t(\bm{g}_i^{t,s} + \bm{b}_i^{t,s}) - \nabla f_i(\bm{\theta}_i^{t,s})\right\|^2\right] + \frac{6\eta_l^2}{\kappa^2}\sum_{t=0}^{T-1}\mathbb{E}\left[\frac{\tau}{r}\sum_{i\in\mathcal{W}}\sum_{s=0}^{\tau-1}\left\|\nabla f_i(\bm{\theta}_i^{t,s}) - \nabla f_i(\bm{\theta}_{t})\right\|^2\right] \\
    & \ \ + \frac{6\tau^2\eta_l^2}{r^2\kappa^2}\sum_{t=0}^{T-1}\mathbb{E}\left\|\sum_{i\in\mathcal{W}}\nabla f_i(\bm{\theta}_{t}) - r\nabla f(\bm{\theta}_{t}) \right\|^2 \\
    &\leq \frac{6T\tau\eta_l^2}{r\kappa^2}\left(\frac{1}{p}(G^2 + \zeta_l^2) + pd\sigma^2\right)  + \frac{6\tau\eta_l^2L^2}{n\kappa^2}\sum_{i=1}^{n}\sum_{t=0}^{T-1}\sum_{s=0}^{\tau-1}\mathbb{E}\left\|\bm{\theta}_i^{t,s} -\bm{\theta}_{t}\right\|^2 +  \frac{6\tau^2T\eta_l^2\zeta_g^2}{\kappa^2}\left(1-\frac{r}{n}\right)\\
    & \leq \frac{6T\tau\eta_l^2}{r\kappa^2} \left(\frac{1}{p}(G^2 + \zeta_l^2) + pd\sigma^2\right) + \frac{30T\tau^3\eta_l^4L^2}{\kappa^2}\left(\frac{1}{p}(G^2 + \zeta_l^2) + pd\sigma^2+ 6\tau \zeta_{g}^2\right) +  \frac{60\eta_l^4\tau^2L^2}{\kappa^2}\sum_{t=0}^{T-1}\mathbb{E}\left\|\nabla f(\bm{\theta}_{t}) \right\|^2\\
    & \ \ 
    + \frac{6\tau^2T\eta_l^2\zeta_g^2}{\kappa^2}\left(1-\frac{r}{n}\right),
\end{align*}
where the first inequality results from \eqref{lemma:aaaa}, the second inequality uses the independence of unbiased stochastic gradients $\{\bm{g}_i^{t,s}, i\in\mathcal{W}, s\in[0,\tau)\}$ and \eqref{lemma:aaaa}, and the last inequality comes from Lemma~\ref{lemma:distance}.
\end{document}